\documentclass[11pt]{article}
\usepackage[top=1in, bottom=1in, left=1in, right=1in]{geometry}
\usepackage[T1]{fontenc}
\usepackage[utf8]{inputenc}
\usepackage[english]{babel}
\usepackage{ifpdf}
\usepackage{lmodern}
\usepackage{microtype}
\usepackage{xspace}
\usepackage{amsmath,amssymb,amsfonts,amsthm,amsopn}
\usepackage{mathtools}
\usepackage{thmtools}
\usepackage{thm-restate}
\usepackage{mathrsfs}
\usepackage{bm}
\usepackage{bbm}
\usepackage{accents}
\usepackage{graphicx}
\usepackage[dvipsnames]{xcolor}
\usepackage{soul}
\usepackage{booktabs,array,tabularx,multirow}
\newcolumntype{M}[1]{>{\centering\arraybackslash}m{#1}}
\usepackage{float}
\usepackage{subfig}
\usepackage{caption}
\usepackage{enumerate}
\usepackage{enumitem}
\usepackage{tikz}
\usepackage[most]{tcolorbox}
\usepackage{appendix}
\usepackage{environ}
\usepackage[ruled,vlined,linesnumbered]{algorithm2e}
\usepackage{algpseudocode}
\usepackage{url}
\usepackage[pagebackref, letterpaper=true, colorlinks=true, urlcolor=blue, linkcolor=blue, citecolor=OliveGreen,]{hyperref}
\usepackage[capitalise,nameinlink]{cleveref}
\numberwithin{equation}{section}

\makeatletter
\providecommand{\headers}[2]{}
\providecommand{\slugger}[5]{}

\providecommand{\email}[1]{\href{mailto:#1}{#1}}

\providecommand{\orcid}[1]{}
\providecommand{\ORCID}[1]{}

\providecommand{\newsiamthm}[2]{\@ifundefined{#1}{\newtheorem{#1}[theorem]{#2}}{}}
\providecommand{\newsiamremark}[2]{\@ifundefined{#1}{\theoremstyle{remark}\newtheorem{#1}[theorem]{#2}\theoremstyle{plain}}{}}
\@ifundefined{keywords}{\newcommand{\keywords}{\@ifnextchar\bgroup{\keywords@cmd}{\keywords@env}}\newcommand{\keywords@cmd}[1]{\par\smallskip\noindent\textbf{Keywords: }#1\par\smallskip}\newcommand{\keywords@env}{\par\smallskip\noindent\textbf{Keywords: }}}{}
\@ifundefined{MSCcodes}{\newcommand{\MSCcodes}{\@ifnextchar\bgroup{\MSCcodes@cmd}{\MSCcodes@env}}\newcommand{\MSCcodes@cmd}[1]{\par\smallskip\noindent\textbf{MSC codes: }#1\par\smallskip}\newcommand{\MSCcodes@env}{\par\smallskip\noindent\textbf{MSC codes: }}}{}
\@ifundefined{appendices}{}{}
\@ifundefined{AMS}{\newcommand{\AMS}{\@ifnextchar\bgroup{\AMS@cmd}{\AMS@env}}\newcommand{\AMS@cmd}[1]{\par\smallskip\noindent\textbf{MSC codes: }#1\par\smallskip}\newcommand{\AMS@env}{\par\smallskip\noindent\textbf{MSC codes: }}}{}

\makeatother

\SetKwInput{KwInput}{Input}
\SetKwInput{KwOutput}{Output}
\SetKwInput{KwData}{Data}
\SetKwInput{KwResult}{Result}
\newcounter{parentAlgoLine}
\SetKwBlock{BEGIN}{}{}
\makeatletter
\NewEnviron{substeps}{\refstepcounter{AlgoLine}\protected@edef\theparentequation{\theequation}\setcounter{parentAlgoLine}{\value{AlgoLine}-1}\setcounter{AlgoLine}{0}\BEGIN{\BODY}\setcounter{AlgoLine}{\value{parentAlgoLine}}\ignorespacesafterend
}
\makeatother

\renewcommand{\epsilon}{\varepsilon}
\renewcommand{\geq}{\geqslant}
\renewcommand{\leq}{\leqslant}

\newcommand\numberthis{\refstepcounter{equation}\tag{\theequation}}
\ifpdf
\DeclareGraphicsExtensions{.eps,.pdf,.png,.jpg}
\else
\DeclareGraphicsExtensions{.eps}
\fi
\newtheorem{theorem}{Theorem}[section]
\newtheorem*{theorem*}{Theorem}

\newtheorem{proposition}[theorem]{Proposition}
\newtheorem*{proposition*}{Proposition}

\newtheorem*{lemma*}{Lemma}

\newtheorem*{conjecture*}{Conjecture}

\newtheorem*{fact*}{Fact}

\newtheorem*{hypothesis*}{Hypothesis}

\theoremstyle{definition}
\newtheorem{definition}[theorem]{Definition}
\newtheorem*{definition*}{Definition}
\newtheorem{model}{Model}[section]

\theoremstyle{remark}

\newtheorem*{claim*}{Claim}
\newtheorem{remark}[theorem]{Remark}
\newtheorem*{remark*}{Remark}

\newtheorem*{observation*}{Observation}
\crefname{lemma}{Lemma}{Lemmas}
\crefname{fact}{Fact}{Facts}
\crefname{theorem}{Theorem}{Theorems}
\crefname{corollary}{Corollary}{Corollaries}
\crefname{claim}{Claim}{Claims}
\crefname{example}{Example}{Examples}
\crefname{problem}{Problem}{Problems}
\crefname{definition}{Definition}{Definitions}
\crefname{model}{Model}{Models}
\crefname{figure}{Figure}{Figures}
\Crefname{figure}{Figure}{Figures}
\crefformat{equation}{\textup{#2(#1)#3}}
\crefrangeformat{equation}{\textup{#3(#1)#4--#5(#2)#6}}
\crefmultiformat{equation}{\textup{#2(#1)#3}}{ and \textup{#2(#1)#3}}{, \textup{#2(#1)#3}}{, and \textup{#2(#1)#3}}
\crefrangemultiformat{equation}{\textup{#3(#1)#4--#5(#2)#6}}{ and \textup{#3(#1)#4--#5(#2)#6}}{, \textup{#3(#1)#4--#5(#2)#6}}{, and \textup{#3(#1)#4--#5(#2)#6}}
\crefname{assumption}{Assumption}{Assumptions}
\crefname{problem}{Problem}{Problems}
\crefname{bottleneck}{Bottleneck}{Bottlencks}
\crefname{extension}{Extension}{Extensions}
\crefname{conjecture}{Conjecture}{Conjectures}
\usepackage[createShortEnv, commandRef=Cref]{proof-at-the-end}
\pgfkeys{/prAtEnd/global custom defaults/.style={
		text link={\par\noindent \textit{Proof.} See the corresponding proof in \hyperref[proof:prAtEnd\pratendcountercurrent]{Appendix~B} on page~\pageref*{proof:prAtEnd\pratendcountercurrent}.\par\smallskip}
}}
\makeatletter
\def\thmt@innercounters{section,equation,theorem}
\makeatother

\makeatletter
\newcommand\footnoteref[1]{\protected@xdef\@thefnmark{\ref{#1}}\@footnotemark}
\makeatother

\newcommand{\norm}[1]{\left\lVert#1\right\rVert}

\newcommand{\inner}[2]{\left\langle #1,#2 \right\rangle}
\newcommand*{\normop}[1]{\norm{#1}_{\mathrm{op}}}

\newcommand{\EX}{\mathbb{E}}
\newcommand{\probP}{\text{I\kern-0.15em P}}

\newcommand\iid{\overset{\mathrm{i.i.d.}}{\sim} }

\newcommand{\minimize}{\mathop{\textrm{minimize}}}

\DeclareMathOperator*{\argmin}{\textbf{arg}\,\textbf{min}}

\newcommand\bdot\bullet

\DeclareMathOperator{\Tr}{Tr}

\DeclareMathOperator{\vectorize}{vec}
\DeclareMathOperator{\rank}{rank}

\newcommand{\range}{\mathrm{range}}

\newcommand{\err}{\mathrm{Err}}

\newcommand{\Det}{\mathrm{Det}}

\newcommand{\SNR}{\mathrm{SNR}}

\newcommand{\sym}{\operatorname{sym}}
\DeclareMathOperator{\diag}{diag}

\newcommand{\dgmm}{\text{DGMM}}

\newcommand{\gmm}{\text{GMM}}

\newcommand{\mm}{\text{MM}}

\newcommand{\Nystrom}{\ifmmode\text{Nystr\"{o}m}\else Nystr\"{o}m\xspace\fi}

\newcommand{\cI}{\mathcal I}

\newcommand{\cM}{\mathcal M}
\newcommand{\cN}{\mathcal N}
\newcommand{\cO}{\mathcal O}
\newcommand{\cP}{\mathcal P}

\newcommand{\cX}{\mathcal X}
\newcommand{\cY}{\mathcal Y}

\newcommand{\scrB}{\mathscr B}

\newcommand{\scrF}{\mathscr F}

\newcommand{\R}{\mathbb R}

\newcommand{\V}{\mathbb{V}}

\newcommand{\sS}{\mathsf{S}}
\newcommand{\sT}{\mathsf{T}}

\newcommand{\bv}{\mathbf{v}}

\newcommand{\bx}{\mathbf{x}}
\newcommand{\by}{\mathbf{y}}
\newcommand{\bz}{\mathbf{z}}

\newcommand{\bX}{\mathbf{X}}
\newcommand{\bY}{\mathbf{Y}}

\newcommand{\bpi}{\bm{\pi}}
\newcommand{\btheta}{\bm{\theta}}

\newcommand{\bmu}{\bm{\mu}}

\newcommand{\frakG}{\mathfrak{G}}

\allowdisplaybreaks
\makeatletter
\newcommand{\doublewidetilde}[1]{{
		\mathpalette\double@widetilde{#1}
}}
\newcommand{\double@widetilde}[2]{
	\sbox\z@{$\m@th#1\widetilde{#2}$}
	\ht\z@=.9\ht\z@
	\widetilde{\box\z@}
}
\makeatother

\newcommand{\proofsketch}[1]{\par\smallskip\noindent\textit{Proof sketch.} #1\par\smallskip}

\begin{document}
	\tolerance=1000
	\emergencystretch=1em

	\title{Diagonally-Weighted Generalized Method of Moments Estimation for Gaussian Mixture Modeling}
	\author{Liu Zhang\thanks{Program in Applied and Computational Mathematics, Princeton University, Princeton, NJ 08544 USA (\href{mailto:lz1619@princeton.edu}{lz1619@princeton.edu}).}
		\and Oscar Mickelin\thanks{Program in Applied and Computational Mathematics, Princeton University, Princeton, NJ 08544 USA (\email{hm6655@princeton.edu}).}
		\and Sheng Xu\thanks{Program in Applied and Computational Mathematics, Princeton University, Princeton, NJ 08544 USA (\email{sxu21@princeton.edu}).}
		\and Amit Singer\thanks{Department of Mathematics and Program in Applied and Computational Mathematics, Princeton University, Princeton, NJ 08544 USA (\href{mailto:amits@math.princeton.edu}{amits@math.princeton.edu}).}
	}
	\date{}
	\maketitle

	\begin{abstract}
With direct computation of moments through order $L$ in $d$ dimensions, standard method of moments (MM) incurs $\Omega(d^L)$ computation and storage costs, and standard generalized method of moments (GMM) additionally requires estimating and inverting an often ill-conditioned covariance matrix of size $\Omega(d^L)$-by-$\Omega(d^L)$ to obtain the optimal weighting matrix. Pereira, Kileel, and Kolda [arXiv:2202.06930 (2022)] avoid the tensor computation cost of MM through implicit moment computation. Despite this progress, incorporating statistically principled weighting while preserving computational efficiency remains an important challenge, particularly at low signal-to-noise ratio (SNR), where sampling variances can differ substantially across moment orders. We introduce diagonally-weighted GMM (DGMM), which assigns one adaptive weight to each moment order through a structured approximation of the optimal weighting matrix while avoiding the covariance inversion, thereby balancing statistical efficiency, computational complexity, and numerical stability. For general parametric models, we prove that DGMM is consistent, asymptotically normal, and near-efficient under explicit structural conditions. For heteroscedastic low-rank Gaussian mixtures, we design a computationally efficient and numerically stable algorithm that extends the implicit moment computation to the DGMM objective, gradients, and weights and incorporates Nystr\"om approximations of the kernel sums. In low-SNR experiments with $d=100$ and $L=4$, where standard MM and GMM are infeasible, controlled weighting comparisons show that DGMM achieves consistently small parameter errors across tested initializations, including the challenging starts where alternative weightings struggle, with shorter median runtimes and fewer iterations. The code and data will be available upon publication at https://github.com/liu-lzhang/dgmm.
		
	\end{abstract}
	
	
	
	\section{Introduction}
	\label{sec:intro}
	Gaussian mixtures have been extensively studied for modeling high-dimensional data with cluster structures, which are ubiquitous in natural science and social science   \cite{mclachlan2004finite, titterington1985statistical}, image processing \cite{permuter2003gaussian}, and speech processing \cite{reynolds1995robust}. 
	The method of moments (MM) dates back to \cite{pearson1894contributions}, where the first few moments were used to estimate the parameters of a mixture of two one-dimensional Gaussians. Many moment-based approaches subsequently extended the classical MM, among which the generalized method of moments (GMM)\footnote{In related literature, GMM is sometimes also used as an abbreviation for Gaussian mixture model. Throughout the paper, GMM will always refer to generalized method of moments and  ``Gaussian mixture'' will always be written in full to avoid confusion.} has been one of the most widely studied since \cite{hansen1982large} first established the result on the optimally-weighted GMM. For textbook treatments of GMM, see, e.g., \cite{newey1994large, hall2013generalized}.
	In the theoretical computer science literature, \cite{dasgupta1999learning} proposed the first provable polynomial-time algorithm for Gaussian mixture parameter estimation, under the assumption that the centers of the Gaussian components are well-separated. Since then, many provable algorithms were proposed, aiming to relax this assumption \cite{sanjeev2001learning, vempala2004spectral, brubaker2008isotropic, dasgupta2013two}. Subsequent breakthroughs that eliminated the separation condition were independently proposed by \cite{kalai2010efficiently} for mixtures of two Gaussians, \cite{belkin2010polynomial} for mixtures of $K$ Gaussians with identical spherical covariance, and \cite{moitra2010settling} for mixtures of $K$ arbitrary Gaussians. In particular, the algorithm in \cite{moitra2010settling} is based on the classical MM. This prompted a renewed interest in moment-based methods and inspired many algorithms on the theme of structured decomposition of higher-order moment tensors, e.g., \cite{hsu2013learning, anandkumar2014tensor, bhaskara2014smoothed, anderson2014more} for isotropic Gaussian mixtures, \cite{guo2022learning} for diagonal Gaussian mixtures, and \cite{ge2015learning, bakshi2022robustly, liu2023robustly} for general Gaussian mixtures. 
	Despite the theoretical advances, deploying moment‐based methods in high-dimensional or higher-order settings is often impractical due to computational bottlenecks. The computational and storage complexity of standard MM and GMM through moment order $L$ in $d$ dimensions is $\Omega(d^L)$. The full GMM weighting additionally requires estimating and inverting an often ill-conditioned moment covariance matrix of size $\Omega(d^L)$-by-$\Omega(d^L)$. \cite{pereira2022tensor} addressed the exponential complexity by developing an implicit moment computation for estimating Gaussian mixtures via MM, without explicitly forming the $d^k$-entry moment tensors. Building on the implicit moment computation introduced in \cite{pereira2022tensor} and the large-sample GMM theory, we address the remaining question on moment weighting: how to weight the moment conditions in a way that is statistically principled without incurring the computational bottlenecks including exponential complexity and numerical instability?
	
	\subsection{Contributions}
	
	We introduce the diagonally-weighted GMM (DGMM), which uses the optimal order-wise diagonal approximation of the full GMM weighting matrix to assign order-specific weights to moment conditions, with the following closed form:
	\begin{align}
		\widehat{W}^{[t]} = \diag(\underbrace{\widehat{w}^{[t]}_1, \dots, \widehat{w}^{[t]}_1}_{d \text{ copies}}, \dots, \underbrace{\widehat{w}^{[t]}_L, \dots, \widehat{w}^{[t]} _L}_{d^L \text{ copies}}), \mkern9mu 
		\widehat{w}_k^{[t]} =   \frac{\Tr\left[ \left(\widehat{S}^{[t]}\right)^{(k, k)}\right]}{\sum_{k^\prime}\norm{\left(\widehat{S}^{[t]}\right)^{(k, k^\prime)}}_F^2},
	\end{align}
	where the diagonal matrix $\widehat{W}^{[t]}$ is constrained to have only $L$ different values on its diagonal, $\widehat{S}^{[t]}$ is the estimated covariance of moment conditions at the current iterate, $(\widehat{S}^{[t]})^{(k,k^\prime)}$ is its $d^k$-by-$d^{k^\prime}$ block of covariances between order-$k$ and order-$k^\prime$ moment conditions. DGMM weighting is adaptive to the sampling variances across moment orders and the moment-matching residuals at the current estimation step, is near-efficient under explicit structural conditions, and can be evaluated quickly by extending the implicit moment computation to \cref{model:heteroscedastic}, thereby achieving a balance among statistical efficiency, computational complexity, and numerical stability. Our contributions cover the statistical, algorithmic, and numerical aspects of DGMM:
	\begin{itemize}
		\item \textbf{Statistical aspect.} For general statistical models under standard GMM identification and regularity conditions, we establish the consistency and asymptotic normality of DGMM and derive its asymptotic covariance in \cref{thm:statistical-properties}, and prove a conditional near-efficiency bound in \cref{thm:near-efficiency}. 
		\item \textbf{Algorithmic aspect.} We apply DGMM to study the parameter estimation problem for low-SNR heteroscedastic low-rank Gaussian mixture (\cref{model:heteroscedastic}) and design an algorithm to obtain the DGMM estimator with computational complexity linear in $N$ and quadratic in the moment orders $L$ for fixed model sizes, without explicitly computing or storing moment tensors. This is achieved by adapting the implicit moment computation machinery in \cite{pereira2022tensor} to the low-rank parameterization in \cref{model:heteroscedastic}. This algorithm significantly reduces the exponential complexity of standard MM and GMM procedures (\cref{thm:computational-complexity-comparison}). 
		\item \textbf{Numerical aspect.} We implement the algorithm and empirically validate the advantages of DGMM compared to alternative moment weighting schemes in terms of parameter estimation error, runtime, and convergence behavior in \Cref{sec:numerical-studies}.
	\end{itemize}	
	
	\begin{model}[Heteroscedastic low-rank Gaussian mixture]    \label{model:heteroscedastic} 
		Fix the number of mixture components $K \geq 2$. Let $h \in [K]$ be a discrete random variable such that $ \pi_j  = \probP(h = j) \in (0,1)$ for $j = 1, \dots, K$ and $\sum_{j=1}^K \pi_j = 1$, where $\pi_j$ is the mixing probability of the $j$-th mixture component. The $j$-th mixture component $\bX_j \in \R^d$ is said to be a \emph{\textbf{low-rank Gaussian}}\footnote{For a more detailed characterization of the p.d.f.~of a low-rank Gaussian, see \Cref{sec:supporting-proofs-model}.} if
		\begin{align}
			\bX_j \sim \cN(\bmu_j, \Sigma_j), \mkern9mu  R_j \coloneqq \rank \Sigma_j  \ll d,
		\end{align}
		where $\bmu_j \in \R^d$ and $\Sigma_j \in \R^{d\times d}$ are the center and the covariance of the $j$-th mixture component. Then the random vector $\bY$ drawn as $\bX_h$ is said to be a \emph{\textbf{heteroscedastic low-rank Gaussian mixture}}\footnote{We use the term ``low-rank Gaussian mixture'' to indicate that the covariance matrix of each Gaussian component in the mixture is of low rank. We note that \cite{lyu2023optimal} also uses the term ``low-rank Gaussian mixture'' but refers to a different model where each observation of the Gaussian mixture is assumed to be matrix-valued and have a planted low-rank structure.} if the covariances are allowed to be non-identical, that is, $\Sigma_i \neq \Sigma_j$ for some $i \neq j$. The mixture variable $\bY$ has the following data generating process:
		\begin{align}
			{\by_n}\ |\ (h_n = j) \sim  \cN(\bmu_j,\Sigma_j), \mkern9mu n = 1, \dots, N,
		\end{align}
		which can be interpreted as a two-step process of picking a component index according to mixing probabilities and sampling a low-rank Gaussian random vector from the picked component. 		
		
		We use the parameterization $\Sigma_{j} = V_{j} {V_{j}}^\top$ with $V_j \in \mathbb{R}^{d\times R_{j}}$ to impose the low-rank constraints on the Gaussian components. The model parameters are collected in a vector concatenating the mixing probabilities, the centers, and the vectorized low-rank factors $V_j$, i.e., $\btheta \coloneq\left[\pi_1; \dots; \pi_K; {\bmu}_1; \dots; {\bmu}_K; \vectorize(V_1); \dots; \vectorize(V_K)\right] \in \Theta \subset \R^{p}$, where $p = K + dK + d \sum_{j=1}^K R_j$.
	\end{model}
	
	\begin{definition}[Gaussian mixture SNR]
		\label{def:average-lowrank-snr}
		For each $j \in [K]$, suppose that the covariance matrix of the $j$-th mixture component $\Sigma_j$ has the following singular value decomposition
		\begin{align}
			\Sigma_j = U_j \Lambda_j U_j^\top, \quad \Lambda_j  = \diag\left({\lambda_j^{(1)}}, \dots, {\lambda_j^{(R_j)}}\right).
		\end{align}
		
		We define the signal-to-noise ratio (SNR) of a given Gaussian mixture variable $\bY$ to be
		\begin{align}
			\label{eq:average-lowrank-snr}
			\SNR(\bY) \coloneq \max_{j \in [K]}	\frac{\norm{{\bmu}_j - \overline{\bmu}}_2^2}{\overline{\lambda}_j}, \quad
			\overline{\bmu} \coloneq \sum_{j=1}^K \pi_j \bmu_j, \quad \overline{\lambda}_j \coloneq \frac{1}{R_j} \sum_{r = 1}^{R_j} \lambda_j^{(r)},
		\end{align}
		where $ \overline{\bmu}$ is the mixture mean and $\overline{\lambda}_j $ the average variance of the $j$-th mixture component.  A given mixture variable $\bY$ is said to be a \emph{\textbf{low-SNR Gaussian mixture}} if $\SNR(\bY) \ll 1$.
	\end{definition}
	
	\begin{remark}
		This definition for Gaussian mixture SNR is a generalization of \cite{katsevich2023likelihood}. When all covariances are identical, full-rank, and isotropic, that is, $\Sigma_j = \sigma^2 I_d$ for all $j$, we have $R_j=d$ and $\overline{\lambda}_j=\sigma^2$, and \cref{eq:average-lowrank-snr} reduces exactly to the definition of mixture SNR in \cite[Definition~2.1 and Eq.~2.8]{katsevich2023likelihood}, which is $\max_{j\in[K]}\frac{\norm{\bmu_j-\overline{\bmu}}_2^2}{\sigma^2}.$
	\end{remark}
	
	We study the parameter estimation problem for \cref{model:heteroscedastic} in the low-SNR regime, when the number of mixture components $K$ and the maximum possible rank $R_{\max} \geq \max\{R_1,\dots,R_K\}$ are given and the individual ranks are unknown. A concrete example is given in \cref{fig:illustration}.
	\begin{figure}[H]
		\captionsetup{justification=justified, singlelinecheck=off}
		\centering
		\subfloat{\includegraphics[width=0.4\textwidth]{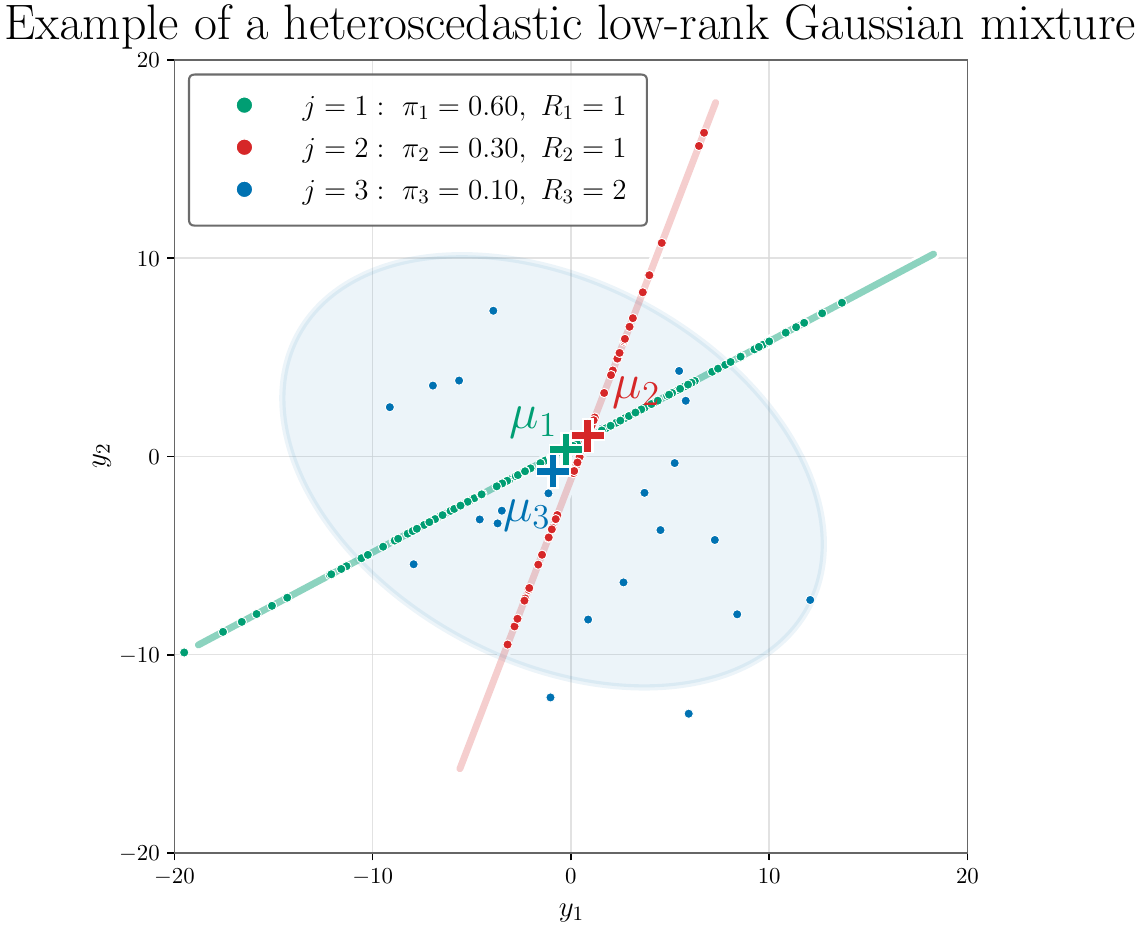}}
		\caption{\small $200$ i.i.d.~observations from an example of \cref{model:heteroscedastic} in the low-SNR regime: $\bY\in \R^{100} \sim 0.6\cN({\bmu}_1, \Sigma_1) + 0.3\cN({\bmu}_2, \Sigma_2) + 0.1\cN({\bmu}_3, \Sigma_3)$, where $\rank \Sigma_1 = \rank \Sigma_2 =1, \rank \Sigma_3 = 2$  (plot shows the first two coordinates). This is a mixture of two Gaussians on one-dimensional support and one Gaussian on two-dimensional support. The covariance/noise scale is large relative to the center/signal scale in the sense of \cref{def:average-lowrank-snr}.}\label{fig:illustration}
	\end{figure} There are several motivations for studying this problem:
	\begin{enumerate}
		\item  \textbf{Low-SNR:} Considering the mixture centers as the signal and the covariances as the noise, we impose the low-SNR condition in order to study the mixture estimation problem in the more challenging low-SNR regime. Related notions of low-SNR mixtures have been explored in recent works including \cite{abas2022generalized, katsevich2023likelihood, fan2023likelihood, fan2024maximum}. Note that the results on statistical properties and  computational complexity of DGMM do not require the low-SNR assumption, but it serves as an important theoretical motivation.
		\item \textbf{Heteroscedasticity:} By allowing the mixture components to have covariance matrices of different ranks and different spectral distributions, we can model data with heterogeneous noise across observations in different clusters. Such data arise when the observations are obtained under different measurement conditions, e.g., econometric data with unobserved heterogeneity \cite{compiani2016using}, spectrophotometric data \cite{cochran1977statistically}, and astronomy data \cite{tamuz2005correcting}. Related heterogeneous noise models have been studied in \cite{hong2023optimally, hong2025optimal} for principal component analysis and \cite{cai2011optimal} for heteroscedasticity detection.
		\item     \textbf{Low-rank:} The low-rank covariance structure in \cref{model:heteroscedastic} is particularly useful for reducing the dimension of the parameter space when the intrinsic dimension is much smaller than the ambient dimension, which is a common assumption in model-based subspace clustering \cite{fraley2002model}, structured graph learning \cite{kumar2020unified}, image processing \cite{gong2019intrinsic, pope2021intrinsic} and diffusion model training \cite{zhang2024emergence, stanczuk2024diffusion}. More concretely, this low-rank assumption is often used in applications in finance \cite{zhou2022covariance}, genomics \cite{ba2023learning}, and spatial statistics \cite{cao2021exploiting}.  
	\end{enumerate}
	
	\subsection{Main results}
	We organize our results on DGMM in \cref{tab:comparison-estimators}.
	\begin{table}[H]
		\centering
		\captionsetup{justification=justified, singlelinecheck=off}
		\begin{tabular}{|M{0.175\linewidth}|M{0.21\linewidth}|M{0.265\linewidth}|M{0.23\linewidth}|}
			\cline{1-4}
			& \textbf{MM}& \textbf{DGMM} &\textbf{GMM}\\ 
			\cline{1-4} 
			\textbf{Weighting} & $W = I_q$, $q = d + \cdots + d^L$  &  $W=$ $\diag( w_1 I_{d} , \dots, w_L I_{d^L})$ & $W = S^{-1} \in \R^{q\times q}$, $q = d + \cdots + d^L$ \\ 
			\cline{1-4}
			\textbf{Statistical properties}  & Consistent, asymptotically normal, suboptimal efficiency  & Consistent, asymptotically normal, near-efficient under structural conditions (\cref{thm:statistical-properties,thm:near-efficiency}) & Consistent, asymptotically normal, optimal efficiency \\ 
			\cline{1-4}
			\textbf{Weight computation} & $\cO(1)$ & $\cO(L^2 K^2 + L K^2 dR_{\max}^2$ $+ NLK + NKdR_{\max})$ (\cref{thm:computational-complexity-W}) & $\cO(Nd^{2L} + d^{3L})$ (\cref{thm:computational-complexity-comparison})\\
			\cline{1-4}
			\textbf{Gradient computation} & $\cO(Nd^{L+1}KR_{\max})$ (\cref{thm:computational-complexity-comparison}) & $\cO(L^2 K^2 + L K^2 dR_{\max}^2$ $+ NLK + NKdR_{\max})$ (\cref{thm:computational-complexity-gradients}) & $\cO(Nd^{L+1}KR_{\max}$ $+ d^{2L+1}KR_{\max})$ (\cref{thm:computational-complexity-comparison}) \\
			\cline{1-4}
			\textbf{Overall computational complexity} & $\cO(Nd^{L+1}KR_{\max})$ (\cref{thm:computational-complexity-comparison}) & $\cO(L^2 K^2 + L K^2 dR_{\max}^2$ $+ NLK + NKdR_{\max}$ $+ NLd m)$ (\cref{thm:computational-complexity-comparison}) & $\cO(Nd^{2L} + d^{3L}$ $+ Nd^{L+1}KR_{\max}$ $+ d^{2L+1}KR_{\max})$ (\cref{thm:computational-complexity-comparison}) \\
			\cline{1-4}
			\textbf{Numerical stability} & Stable  & Stable & Unstable, due to $\left(\widehat{S}^{[t]}\right)^{-1}$ in \cref{eq:o-gmom-multistep}\\
			\cline{1-4}
		\end{tabular}
		\caption{Comparison of moment-based estimators. $N$: the sample size, $d$: the number of dimensions, $K$: the number of mixture components, $R_{\max} \geq \max(R_1, \dots, R_K)$: the given maximum possible rank, $L$: the highest moment order, $m$: the number of landmarks for Nystr\"om approximation in \cref{thm:computational-complexity-inner-prod-sum}.}\label{tab:comparison-estimators}
	\end{table}
	\subsection{Outline}
	We give the preliminaries in \Cref{sec:preliminaries} and present the proposed DGMM estimator in \Cref{sec:proposed-dgmom}. In \Cref{sec:statistical-computational}, we detail our results on the statistical properties of the DGMM estimator and the computational complexity of our algorithm for obtaining the DGMM estimator for \cref{model:heteroscedastic}. In \Cref{sec:numerical-studies}, we report the numerical results.
	\section{Preliminaries}
	\label{sec:preliminaries}
	In this section, we introduce the relevant preliminary background on tensors, moments and cumulants, the model, moment-based methods including MM and GMM. 
	\subsection{Tensors}
	
	The definitions are based on standard references, e.g., \cite{comon2008symmetric, kolda2009tensor}. 
	\begin{definition}[Tensors, {\cite{kolda2009tensor}}]
		\label{def:tensors}
		Let $\V_1, \dots, $ $\V_k$ be finite-dimensional vector spaces, each with its own coordinate system. 
		An element of the tensor product of $\V_1, \dots, $ $\V_k$ is called a \emph{tensor} and we say that the tensor is of order $k$ and dimensions $(\dim{\V_1}, \dots, \dim{\V_k})$.
	\end{definition}
	
	For the purpose of our paper, we consider $\cX \in \sT^k(\R^d)\coloneq \underbrace{\R^d \otimes \cdots \otimes \R^d}_{k \text{ copies }},$ which is the set of all real-valued tensors of order $k$ and dimensions $(d, \dots, d)$. 
	
	\begin{definition}[Inner product and norm on $\sT^k(\R^d)$, {\cite{kolda2009tensor}}]
		\label{def:tensor-inner-product}
		The \emph{inner product} of two tensors $\cX, \cY \in \sT^k(\R^d)$ is 
		\begin{align}
			\inner{\cX}{\cY} = \sum_{i_1 = 1}^d \cdots \sum_{i_k = 1}^d \cX_{i_1 \dots i_k} \cY_{i_1 \dots i_k}.
		\end{align}
		The \emph{norm} of a tensor $\cX \in \sT^k(\R^d)$ is 
		\begin{align}
			\label{eq:tensor-norm}
			\norm{\cX}  = \sqrt{\inner{\cX}{\cX}} = \sqrt{\sum_{i_1 = 1}^d \cdots \sum_{i_k = 1}^d \cX^2_{i_1 \dots i_k}}.
		\end{align}
	\end{definition}
	
	\begin{definition}[Symmetric tensors, {\cite{comon2008symmetric}}]
		\label{def:symmetric-tensors}
		An order-$k$ tensor $\cX \in \sT^k(\R^d)$ is \emph{symmetric} if, for any permutation of the $k$ indices, its entries remain unchanged, i.e., 
		\begin{align}
			\cX_{i_{\sigma(1)} \dots i_{\sigma(k)}} = \cX_{i_1 \dots i_k}, \mkern9mu i_1, \dots, i_k \in [d],
		\end{align}
		for all permutations $\sigma \in \frakG_k$, where $\frakG_k$ is the symmetric group of permutations on the indices $\{1, \dots, k\}$. $\sS^k(\R^d)$ denotes the set of symmetric tensors in $\sT^k(\R^d)$. 
	\end{definition}
	
	\begin{definition}[Symmetrization operator, {\cite{comon2008symmetric}}]
		The symmetrization operator $\sym: \sT^k(\R^d) \to \sS^k(\R^d) $ is the linear operator that acts on $\cX \in \sT^k(\R^d)$ in the following way 
		\begin{align}
			\left(\sym(\cX)\right)_{i_1 \dots i_k} = \frac{1}{k!} \sum_{\sigma\in\frakG_k} \cX_{i_{\sigma(1)} \dots i_{\sigma(k)}}.
		\end{align}
	\end{definition}
	
	By \cite[Proposition 3.3]{comon2008symmetric}, an equivalent reformulation of symmetric tensors is via the symmetrization operator: an order-$k$ tensor $\cX \in \sT^k(\R^d)$ is symmetric if and only if  $\sym(\cX) = \cX$.
	
	\subsection{Moments, cumulants, Bell polynomials}
	We present the definitions of moments, cumulants, and Bell polynomials, as well as how they are related. Throughout the paper, we assume that every random variable is defined on a suitable probability space $\left(\Omega, \scrF, \probP\right)$. 
	\begin{definition}[Moments]
		\label{def:moments}
		
		Suppose $X$ is a real-valued random variable. Then the $k$-th \emph{uncentered moment} of $X$ is defined as 
		\begin{align}
			\cM^{(k)}(X) = \EX\left[X^k\right].
		\end{align} 
		
	\end{definition}
	
	\begin{definition}[Cumulants] 
		\label{def:cumulants}
		
		The $k$-th \emph{cumulant} $\kappa_k(X)$ of a random variable $X$ is defined recursively by the following relation between moments and cumulants:
		\begin{align}
			\EX[X^k] = \sum_{\pi \in \Upsilon_k} \prod_{S \in \pi} \kappa_{|S|} (X),
		\end{align}
		where $\Upsilon_k$ is the set of all partitions of the set $[k]$. A partition $\pi \in \Upsilon_k$ is a set of disjoint non-empty subsets $S$ such that $\bigsqcup_{S\in \pi} S = [k]$. 
		
	\end{definition}
	
	\begin{definition}[Complete Bell polynomials, {\cite{bell1934exponential}} (cf.~{\cite[Def.~11.1]{charalambides2018enumerative}})]
		\label{def:complete-bell}
		The \emph{complete Bell polynomials} in the variables $x_1, \dots, x_k$ are defined by the sum \begin{equation}
			B_k\left(x_1, \dots , x_k\right) = \sum_{(j_1, \dots, j_k)\in \cP_k} \frac{k!}{j_1! j_2 ! \cdots j_k!}\left(\frac{x_1}{1!}\right)^{j_1} \left(\frac{x_2}{2!}\right)^{j_2} \cdots \left(\frac{x_k}{k!}\right)^{j_k}, 
		\end{equation}
		where $\cP_k = \left\{(j_1, \dots, j_k) : j_1 + 2j_2 + \cdots + kj_k = k, j_1, \dots, j_k \in \mathbb{Z}_{\geq 0} \right\}.$
	\end{definition}
	
	\begin{definition}[Partial Bell polynomials, {\cite{bell1934exponential}} (cf.~{\cite[Def.~11.2]{charalambides2018enumerative}})]
		\label{def:partial-bell}
		The \emph{partial Bell polynomials} in the variables $x_1, \dots, x_k$ of degree $l$ are defined by the sum
		\begin{equation}
			B_{k,l}\left(x_1, \dots , x_k\right) = \sum_{(j_1, \dots, j_k)\in \cP_{k,l}} \frac{k!}{j_1! j_2 ! \cdots j_k!}\left(\frac{x_1}{1!}\right)^{j_1} \left(\frac{x_2}{2!}\right)^{j_2} \cdots \left(\frac{x_k}{k!}\right)^{j_k}, 
		\end{equation}
		where $\cP_{k,l} = \left\{ (j_1, \dots, j_k): j_1 + 2j_2 + \cdots + kj_k = k, j_1 + j_2 + \cdots + j_k = l, j_1, \dots, j_k \in \mathbb{Z}_{\geq 0}  \right\},$ i.e., the summation is taken over the partitions of $k$ \textit{into $l$ parts}. 
	\end{definition}
	
	The following proposition relates moments, cumulants, and Bell polynomials.
	\begin{proposition}[{\cite[p.~160]{comtet1974advanced}}, {\cite[Proposition 3.3.1]{peccati2011wiener}}]
		\label{prop:mom-cum-bell}
		Let $X$ be a random variable such that $\EX\left[ |X|^k \right] < \infty$. For $ l = 1, \dots, k$, let $\cM^{(l)} (X) = \EX\left[X^l\right]$ be the $l$-th moment and $\kappa^{(l)}(X)$ be the $l$-th cumulant. Then 
		\begin{align}
			\cM^{(k)}(X) &= B_k(\kappa^{(1)}(X), \dots, \kappa^{(k)}(X))\\
			\kappa^{(k)}(X) &= \sum_{l = 1}^k (-1)^{l-1} (l-1)! B_{k, l}\left(\cM^{(1)}(X), \dots, \cM^{(k-l+1)}(X)\right).
		\end{align}
	\end{proposition}
	
	\subsection{Method of moments (MM) and generalized method of moments (GMM)}
	\label{sec:mm-gmm}
	
	To apply moment-based methods, including MM and GMM, to estimate the unknown parameter $\btheta \in \Theta \subset \R^p$ of \cref{model:heteroscedastic}, we start from defining the \emph{moment function} 
	$g: \Theta \times \R^d \to \R^q$ by
	\begin{align}\label{eq:moment-function}
		g(\btheta; \by) \coloneq &\left[ \vectorize(\cM^{(1)}\left({\btheta}\right) - {\by}); \cdots; \vectorize(\cM^{(L)}\left({\btheta}\right) - {\by}^{\otimes L})\right] \in \R^{q},
	\end{align}
	where $q = d + d^2 + \cdots + d^L$, $L$ is the highest moment order, and $\cM^{(k)}({\btheta})$ denotes the $k$-th population moment that can be computed by \cite[Theorem~4.1]{pereira2022tensor}:
	\begin{align*}\label{eq:k-th-pop-moment}
		\cM^{(k)}(\btheta) = \sum_{j=1}^K \sum_{l=0}^{\lfloor{k/2}\rfloor} \pi_j C_{k, l} \sym \left({{\bmu}_j}^{\otimes(k-2l)} \otimes (V_j {V_j}^\top)^{\otimes l}\right) \in \sS^k (\R^d), \mkern9mu C_{k,l} = \frac{k!}{(k-2l)!l!2^l}.\numberthis
	\end{align*}
	
	Following the standard practice (see \cite{hansen1982large, newey1994large}), we restrict the parameter space $\Theta$ to be a compact set containing the true parameters. For the purpose of this paper, we assume the standard global and local GMM identification conditions (see, e.g., \cite{hansen1982large, newey1994large}), that is, there exists a suitable $L$ such that \cref{model:heteroscedastic} is
	\begin{enumerate}
		\item \emph{globally identifiable}, i.e., $\EX[g(\btheta; \by_n)] = \bm{0} \text{ if and only if } {\btheta} = {\btheta}^\star$, and
		\item  \emph{locally identifiable}, i.e., $\rank(G(\btheta^\star)) = p$, where $G(\btheta^\star) = \EX\left[\nabla_{{\btheta}} g({\btheta}^\star; {\by_n})\right]$.
	\end{enumerate}
	
	Given $N$ independent observations $y_1, \dots, y_N \in \R^d$, the corresponding vector of \emph{sample moment conditions} is 
	\begin{align}\label{eq:mom-conditions}
		\overline{g}_N\left({\btheta}\right) \coloneq \frac{1}{N} \sum_{n=1}^N g(\btheta; \by_n)  \in \R^{q}.
	\end{align}
	
	If $q = p$, then one may set up a system of polynomial equations and solve for $\btheta$ exactly:
	\begin{align}
		\overline{g}_N\left({\btheta}\right)  = \bm{0}  \in \R^{q}.
	\end{align}
	
	When $q>p$, the model is overidentified, that is, there are more sample moment conditions than the unknown parameters. In this case, one instead obtains an estimator that makes the sample moment conditions as close as possible to zero according to some metric or measure of distance. More specifically, the MM estimator, $\widehat{{\btheta}}^{(\mm)}$, is obtained by solving a moment-matching optimization problem in $\btheta$:
	\begin{align}\label{eq:mm-optimization}
		\minimize_{{\btheta} \in \Theta} & \mkern9mu \norm{\overline{g}_N \left({\btheta}\right)}_2^2 = \sum_{k=1}^L \norm{\cM^{(k)}(\btheta) - \frac{1}{N} \sum_{n=1}^N \by_n^{\otimes k}}^2,\\
		\textbf{subject to} & \mkern9mu 0<\pi_1, \dots, \pi_K<1, \mkern9mu\sum_{j=1}^K \pi_j=1.
	\end{align}
	
	The GMM estimator, $\widehat{{\btheta}}^{(\gmm)}$ replaces the Euclidean metric by a weighted version:
	\begin{equation}
		\begin{aligned}\label{eq:gmm-optimization}
			\minimize_{{\btheta} \in \Theta} & \mkern9mu \overline{g}_N\left({\btheta}\right)^\top W \overline{g}_N\left({\btheta}\right),\\
			\textbf{subject to} & \mkern9mu 0<\pi_1, \dots, \pi_K<1, \mkern9mu\sum_{j=1}^K \pi_j=1,
		\end{aligned} 
	\end{equation}
	for symmetric positive semi-definite weighting matrix $W\in \R^{q\times q}$.  \cite{hansen1982large} shows that the GMM estimator achieves optimal asymptotic efficiency when $W = c S^{-1}$, for any $c>0$ and 	$S = \EX \left[g({\btheta}^{\star}; {\by_n}) g({\btheta}^{\star}; {\by_n})^\top\right]$.
	When $W= I$, $\widehat{{\btheta}}^{(\gmm)}$ is equivalent to the MM estimator, $\widehat{{\btheta}}^{(\mm)}$. In practice, $\widehat{{\btheta}}^{(\gmm)}$ is obtained via a multi-step estimation procedure, where at each estimation step, a consistent estimator of $S$ and a consistent estimator of $\btheta$ are obtained alternately:
	\begin{equation}
		\begin{aligned}\label{eq:o-gmom-multistep}
			\widehat{S}^{[t]} =& \frac{1}{N}\sum_{n=1}^N g(\widehat{{\btheta}}^{[t-1]}; {\by_n}) g(\widehat{{\btheta}}^{[t-1]}; {\by_n})^\top,\\
			\widehat{{\btheta}}^{[t]} = & \quad\argmin_{{\btheta} \in \Theta} \mkern9mu \overline{g}_N\left({\btheta}\right)^\top \widehat{W}^{[t]}\overline{g}_N\left({\btheta}\right), \mkern9mu \widehat{W}^{[t]} = \left(\widehat{S}^{[t]}\right)^{-1},
		\end{aligned} 
	\end{equation}
	which terminates when $t = T$ or $\norm{\widehat{{\btheta}}^{[t]} -  \widehat{{\btheta}}^{[t-1]}} < \varepsilon_{{\btheta}}$ with stopping criteria constants $T, \varepsilon_{{\btheta}}$.
	
	\section{Proposed DGMM estimator}
	\label{sec:proposed-dgmom}
	
	A key computational bottleneck in deploying MM and GMM in practice is that directly computing the sample moment conditions $\overline{g}_N({\btheta})$ in \cref{eq:mom-conditions} and the optimization objective $\overline{g}_N\left({\btheta}\right)^\top W \overline{g}_N\left({\btheta}\right)$ in \cref{eq:gmm-optimization} through moment order $L$ in $d$ dimensions incurs computational and storage cost of $\Omega(d^L)$. This quickly becomes prohibitive for high-dimensional data or when higher-order moments are required. In addition, the full GMM weighting in \cref{eq:o-gmom-multistep} requires estimating and inverting a $q\times q$ covariance matrix. When $d$ and $L$ are large, the number of moment conditions $q \geq d^L$, or at least $d^L / L!$ after removing redundant moment conditions, is large relative to the sample size $N$, making this inversion step numerically unstable. More specifically, when $q > N$, $\widehat{S}^{[t]}$ is not even invertible. When $q < N$ but $q/N$ is not negligible, $\widehat{S}^{[t]}$ is invertible but has a large condition number. Prior works including \cite{ledoit2004well, roodman2009note} have discussed this issue more extensively and pointed out that it is usually unrealistic in practice to find large enough sample size $N$ to make the ratio $q/N$ negligible. As a result, the full GMM weighting may end up degrading the parameter estimation performance, despite achieving the optimal asymptotic efficiency. 
	
	To address these computational bottlenecks, we propose the diagonally-weighted GMM (DGMM), where instead of the full GMM weighting, we restrict to the class of diagonal matrices of the form $\mathcal{W} = \{ \diag(\underbrace{w_1, \dots, w_1}_{d \text{ copies}}, \dots, \underbrace{w_L, \dots, w_L}_{d^L \text{ copies}})\in \R^{q\times q}\}$ and find the optimal approximation of the theoretically optimal weighting matrix within $\mathcal{W}$:
	\begin{equation}\label{eq:diagonal-approximate-inverse-opt}
		W = \argmin_{W \in \mathcal{W}} \norm{ W S - I}^2_F,
	\end{equation}
	where $\norm{\cdot}_F$ denotes the Frobenius norm. This particular form of weighting is motivated by the following considerations: 
	\begin{enumerate}
		
		\item  \textbf{Structured weighting without covariance inversion.} 
		
		The $S$ matrix describes the second-order interactions in the sample moment conditions. By using the structured approximate-inverse, DGMM circumvents the numerically unstable inversion step in \cref{eq:o-gmom-multistep} while preserving the most important order-wise interactions, at the cost of some loss in asymptotic efficiency that we will precisely quantify and prove a conditional near-efficiency bound in \cref{thm:near-efficiency}. This efficiency gap between DGMM and the optimally-weighted GMM is reduced when the cross-order correlations across sample moment conditions of different orders are weak and when the sample moment conditions of the same order share comparable sampling variances.
		
		\item \textbf{Weighting across moment orders is especially relevant in the low-SNR regime.} For a fixed sample size, the sampling variances of the sample moment conditions can grow rapidly with the moment order (see, e.g., \cite[Lemma 5]{wu2020optimal}). The differences in sampling variances become especially prominent in the low-SNR regime. With identity weighting, the noisiest moment order may dominate the objective, resulting in suboptimal statistical efficiency and numerical conditioning issues. DGMM weighting mitigates this issue through the order-specific weights in \cref{eq:diagonal-approximate-inverse-opt}. This is conceptually aligned with \cite[Theorem~2.2]{katsevich2023likelihood}, which shows that, in the low-SNR identical-covariance setting, the population likelihood expansion is organized as a hierarchy of moment-matching discrepancies indexed by moment order.
		
		\item \textbf{Estimating fewer weights reduces finite-sample bias.} 
		Estimating the GMM weights can produce substantial finite-sample bias. This phenomenon has been documented in many previous studies in econometrics, e.g., \cite{hansen1996finite, altonji1996small, neweysmith2004higherorder}. Therefore, estimating only $L$ weights empirically reduces the finite-sample bias.
		
		\item \textbf{Computational complexity considerations.} An unrestricted diagonal version of \cref{eq:diagonal-approximate-inverse-opt} estimates $q = \Omega(d^L)$ coordinate-specific weights, whereas DGMM uses one weight for each moment order. This reduces the number of weights from $q$ to $L$ and allows for implicit moment computation speed-up, as shown in \cref{thm:computational-complexity-norm}-\cref{thm:computational-complexity-comparison}.
	\end{enumerate}
	
	Intuitively, DGMM weighting is a structured regularization of the full GMM weighting. It prioritizes the order-wise weighting to account for potentially large differences in sampling variances across moment orders. In this way, DGMM addresses the statistical-computational trade-offs in moment-based estimation. More specifically, it trades some potential efficiency gain from additional weighting for substantially lower computational and storage costs, improved numerical stability, and reduced finite-sample bias from weight-estimation error.
	
	To obtain a closed-form solution for ${w_k}$ in \cref{eq:diagonal-approximate-inverse-opt}, we use the first-order optimality condition: 
	\begin{equation}
		\begin{split}
			0 &=  \nabla_{w_k} \sum_{i \in \cI_k} \sum_{j=1}^{q} \left( w_k S_{ij} - I_{ij}\right)^2  =  2\sum_{i \in \cI_k }\sum_{j=1}^{q} \left( w_k S_{ij} - I_{ij}\right) \nabla_{w_k} (w_k S_{ij})\\
			&=  2\sum_{i \in \cI_k } \sum_{j=1}^{q} \left( w_k S_{ij} - I_{ij}\right)  S_{ij} = 2\sum_{i \in \cI_k } \left(\sum_{j=1}^{q}  w_k S_{ij}^2 -  S_{ii}\right),
		\end{split}
	\end{equation}
	where $\cI_k$ denotes the set of indices in $S$ for the $k$-th moment, which gives the desired weights
	\begin{align}
		\label{eq:w_k-direct}
		w_k = \frac{\sum_{i \in \cI_k} S_{ii}}{\sum_{i \in \cI_k }\sum_{j=1}^{q}S_{ij}^2} =  \frac{\Tr\left[ S^{(k, k)}\right]}{\sum_{k^\prime}\norm{S^{(k, k^\prime)}}_F^2}, \quad k = 1, \dots, L,
	\end{align}
	where $S^{(k, k^\prime)}$ is the $d^k$-by-$d^{k^\prime}$ submatrix of $S$, corresponding to the cross-order covariance between the $k$-th order and $k^\prime$-th order moment conditions
	\begin{align}
		S^{(k, k^\prime)} \coloneq \left(S_{ij}\right)_{i\in \cI_k, j\in \cI_{k^\prime}}= \EX \left[ g^{(k)}(\btheta^\star; \by_n) g^{(k^\prime)}(\btheta^\star; \by_n)^\top \right].
	\end{align}
	
	We obtain the DGMM estimator $\widehat{{\btheta}}^{(\dgmm)}$ via a multi-step procedure similar to \cref{eq:o-gmom-multistep}:
	\begin{equation}
		\begin{aligned}\label{eq:d-gmom-multistep}
			\widehat{S}^{[t]} =& \frac{1}{N}\sum_{n=1}^N g(\widehat{{\btheta}}^{[t-1]}; {\by_n}) g(\widehat{{\btheta}}^{[t-1]}; {\by_n})^\top,\\
			\widehat{{\btheta}}^{[t]} = & \quad \argmin_{{\btheta} \in \Theta} \mkern9mu Q^{[t]}_N(\btheta) \coloneq \overline{g}_N\left({\btheta}\right)^\top \widehat{W}^{[t]}\overline{g}_N\left({\btheta}\right),\\
			& \quad \widehat{W}^{[t]} = \diag(\underbrace{\widehat{w}^{[t]}_1, \dots, \widehat{w}^{[t]}_1}_{d \text{ copies}}, \dots, \underbrace{\widehat{w}^{[t]}_L, \dots, \widehat{w}^{[t]} _L}_{d^L \text{ copies}}), \mkern9mu 
			\widehat{w}_k^{[t]} =   \frac{\Tr\left[ \left(\widehat{S}^{[t]}\right)^{(k, k)}\right]}{\sum_{k^\prime}\norm{\left(\widehat{S}^{[t]}\right)^{(k, k^\prime)}}_F^2},\\
			& \textbf{subject to} \mkern9mu 0<\pi_1, \dots, \pi_K<1, \quad \sum_{j=1}^K \pi_j=1.
		\end{aligned}
	\end{equation}
	
	\section{Statistical properties and computational complexity}
	\label{sec:statistical-computational}
	Since DGMM weighting applies broadly to other parametric models beyond \cref{model:heteroscedastic}, we state our results on the asymptotic properties of the DGMM estimator in the general form in \cref{thm:statistical-properties} and \cref{thm:near-efficiency}. Subsequent computational complexity results are specialized to the context of \cref{model:heteroscedastic}. 
	\begin{theoremEnd}[end]{theorem}[Statistical properties of ${\widehat{\btheta}}^{(\dgmm)}$] 
		\label{thm:statistical-properties}
		Under the standard GMM global and local identification and regularity conditions, the DGMM estimator ${\widehat{\btheta}}^{(\dgmm)}$ has the following statistical properties:
		\begin{enumerate}
			\item (Consistency)  ${\widehat{\btheta}}^{(\dgmm)} \xrightarrow{p} \btheta^\star$ as $N\to \infty$, where $\xrightarrow{p}$ denotes convergence in probability. 
			\item (Asymptotic normality) $ 
			\sqrt{N}({\widehat{\btheta}}^{(\dgmm)} - \btheta^\star)\xrightarrow{d}\cN(0, V^{(\dgmm)})$, where $\xrightarrow{d}$ denotes convergence in distribution. 
			\item (Asymptotic covariance) 
			With the weighting given in \cref{eq:w_k-direct}, the asymptotic covariance matrix of the DGMM estimator is 
			\begin{align*}
				&V^{(\dgmm)} = \\
				& \left[ \sum_{k=1}^L w_k {G^{(k)}}^\top G^{(k)}\right]^{-1} \left[\sum_{k=1}^L  \sum_{k^\prime=1}^L w_k w_{k^\prime} {G^{(k)}}^\top S^{(k, k^\prime)} G^{(k^\prime)}\right] \left[ \sum_{k=1}^L w_k {G^{(k)}}^\top G^{(k)}\right]^{-1},\numberthis  	\label{eq:dgmm-asymptotic-efficiency}
			\end{align*}
			where $G^{(k)}$ is the $d^k$-by-$p$ submatrix of $G \coloneq G(\btheta^\star) = \EX\left[\nabla_{{\btheta}} g({\btheta}^\star; {\by_n})\right]$ corresponding to the $k$-th order moment conditions.
		\end{enumerate}
	\end{theoremEnd}
	
	\begin{proofEnd}
		Define the \emph{population objective function}:
		\begin{align}
			Q(\btheta) \coloneq \EX\left[g\left({\btheta}; \by_n \right)\right]^\top W \EX\left[g\left({\btheta}; \by_n \right)\right], \mkern9mu W = \diag(\underbrace{w_1, \dots, w_1}_{d \text{ copies}}, \dots, \underbrace{w_L, \dots, w_L}_{d^L \text{ copies}}).
		\end{align}
		
		We check the standard regularity conditions for consistency, in e.g.~\cite[Theorem~2.6]{newey1994large}:
		\begin{enumerate}[label=(\alph*)]
			\item The sample data is \text{i.i.d.}, assumed in the specifications of \cref{model:heteroscedastic}.
			\item The parameter space $\Theta$ is compact and contains $\btheta^\star$ as an interior point, which is a standard assumption.
			\item The moment function $g(\btheta; \by_n)$ in \cref{eq:moment-function} depends polynomially on the parameter $\btheta$ and is therefore continuous at each $\btheta \in \Theta$.
			\item Since $\Theta$ is compact and all moments of any finite Gaussian mixture are finite, we have $\EX[\sup_{\btheta\in \Theta} \norm{g(\btheta; \by_n)}]< \infty$.
			\item The model is globally identifiable, which is a standard assumption (see \Cref{sec:mm-gmm}).
			\item The DGMM weighting matrices in \cref{eq:d-gmom-multistep} satisfy $\widehat{W}^{[t]} \xrightarrow{p}  W$. $W$ is positive-definite due to the following: The numerator in \cref{eq:d-gmom-multistep}  $\Tr\left[S^{(k,k)}\right] = \EX[\norm{g^{(k)}(\btheta^\star; \by_n)}^2 ] > 0$.
			The denominator in \cref{eq:d-gmom-multistep} is positive: every term under the double sum is a square, with at least one strictly positive entry.
		\end{enumerate}
		
		Therefore, we get that  ${\widehat{\btheta}}^{[t]}\xrightarrow{p} \btheta^\star $, which proves the consistency of ${\widehat{\btheta}}^{(\dgmm)}$.
		
		To show that ${\widehat{\btheta}}^{(\dgmm)}$ is asymptotically normal, in addition to the above regularity conditions for consistency, we need to check the following conditions:
		\begin{enumerate}[label=(\alph*), resume]
			\item The moment function $g(\btheta; \by_n)$ in \cref{eq:moment-function} is continuously differentiable in a neighborhood of $\btheta^\star$ with probability approaching one as $N\to\infty$. 
			\item $\EX[\norm{g(\btheta^\star; \by_n)}^2] < \infty$ and $\EX[\sup_{\btheta\in \cN(\btheta^\star)} \norm{\nabla_{\btheta} g(\btheta; \by_n)}]< \infty$, which follows from the definitions of \cref{model:heteroscedastic} and \cref{eq:moment-function}. 
			\item $G(\btheta^\star)^\top W G(\btheta^\star)$ is non-singular, which follows from the fact that $W$ is positive-definite and the standard assumption discussed in \Cref{sec:mm-gmm} that the model is locally identifiable, that is, $\rank(G(\btheta^\star)) = p$.
		\end{enumerate}
		
		By the standard asymptotic normality result, in e.g.~\cite[Theorem~3.4]{newey1994large}, if conditions (a)-(i) are satisfied, then for any GMM estimator $\widehat{\btheta}$,
		\begin{align}
			\sqrt{N} \left(\widehat{\btheta} - \btheta^\star\right) \xrightarrow{d} \cN\left(\bm{0}, \left(G^\top W G\right)^{-1} G^\top W S W G  \left(G^\top W G\right)^{-1} \right),
		\end{align}
		where we write $G = G(\btheta^\star)$ for notational convenience.  For the DGMM estimator, since the weighting matrix $W$ is block-diagonal, we can express the asymptotic variance $S$ in block form to highlight how each moment group contributes to the overall asymptotic variance. To do this, we partition the Jacobian matrix $G$ into $L$ blocks, $G = \left[{G^{(1)}}^\top,\dots,{G^{(L)}}^\top\right]^\top$, so that ${G^{(k)}}$ is the $d^k$-by-$p$ submatrix of $G$ corresponding to the $k$-th moment. We also partition $S$ into blocks such that $S^{(k, k^\prime)}$ is the $d^k$-by-$d^{k^\prime}$ submatrix of $S$, consisting of the cross-order covariance between the $k$-th and $k^\prime$-th sample moment conditions, for $k, k^\prime = 1, \dots, L$, i.e.,
		\begin{align}
			S^{(k, k^\prime)} \coloneq \left(S_{ij}\right)_{i\in \cI_k, j\in \cI_{k^\prime}} = \EX \left[ g^{(k)}(\btheta^\star; \by_n) g^{(k^\prime)}(\btheta^\star; \by_n)^\top \right].
		\end{align}
		Then we get the form of the asymptotic covariance matrix of $\widehat{\btheta}^{(\dgmm)}$ in \cref{eq:dgmm-asymptotic-efficiency}.
	\end{proofEnd}
	
	\proofsketch{The proof uses standard multi-step GMM asymptotic theory. We first check the necessary conditions and then apply \cite[Theorem~2.6]{newey1994large} and \cite[Theorem~3.4]{newey1994large} to get consistency and asymptotic normality, respectively. 
		Multiplication in moment blocks then gives the asymptotic efficiency in  \cref{eq:dgmm-asymptotic-efficiency}. 
	}
	
	\begin{theoremEnd}[end]{theorem}[Conditional near-efficiency of the DGMM estimator]
		\label{thm:near-efficiency}
		Under the assumptions of \cref{thm:statistical-properties}, suppose that there exists $0\leq \eta<1$ such that, for all $k=1,\ldots,L$, 
		\begin{align}
			\label{eq:near-efficiency-condition}
			\normop{w_k S^{(k,k)} - I_{d^k}} + \sum_{k^\prime \neq k}	\sqrt{w_k w_{k^\prime}} \normop{S^{(k,k^\prime)}} \leq \eta,
		\end{align}
		where  $\normop{\cdot}$ denotes the operator norm.
		Then
		\begin{align}
			V^{(\gmm)} \preceq V^{(\dgmm)} \preceq \left(1+\frac{\eta^2}{1-\eta^2}\right) V^{(\gmm)}, \quad V^{(\gmm)}=(G^\top S^{-1}G)^{-1}.
		\end{align}
	\end{theoremEnd}
	
	\begin{proofEnd}
		
		As shown in \cite[Theorem~3.2]{hansen1982large}, in the class of GMM estimators, $V^{(\gmm)} = \left(G^\top S^{-1} G\right)^{-1}$ is the minimum asymptotic covariance within the class of GMM estimators, which directly implies that $V^{(\dgmm)} -V^{(\gmm)} \succeq 0$. We now obtain the upper bound on $V^{(\dgmm)}$. Let $W$ be the DGMM weighting matrix in \cref{eq:w_k-direct}.
		We first apply a change of coordinate to get $\widetilde{S}  = W^{\frac{1}{2}} S W^{\frac{1}{2}}$. Let $E = \widetilde{S}  - I_q$. Then
		\begin{align}
			E^{(k, k)} = w_k S^{(k,k)} - I_{d^k}, \quad E^{(k, k^\prime)} = \sqrt{w_k w_{k^\prime}} S^{(k, k^\prime)} (k\neq k^\prime).
		\end{align} 
		Since $E$ is a symmetric matrix, $E^{(k, k^\prime)}  = {E^{(k^\prime, k)}}^\top$ and so $\normop{E^{(k, k^\prime)}} = \normop{E^{(k^\prime, k)}}$. For an arbitrary vector $\bx = \left[\bx_1; \cdots; \bx_L \right] \in \R^q$ where $\bx_k \in \R^{d^k}$, by applying Cauchy-Schwarz, AM-GM inequality, and the symmetry of $E$, we get 
		\begin{align*}
			|\bx^\top E \bx| \leq \sum_{k, k^\prime} \normop{E^{(k, k^\prime)}} \norm{\bx_k} \norm{\bx_{k^\prime}} &\leq \frac{1}{2} \sum_{k, k^\prime} \normop{E^{(k, k^\prime)}} \left( \norm{\bx_k}^2 +  \norm{\bx_{k^\prime}}^2 \right) \\ 
			&= \sum_{k} \left( \sum_{k^\prime} \normop{E^{(k, k^\prime)}} \right) \norm{\bx_k}^2 \leq \eta \norm{\bx}^2,
		\end{align*}
		where the last step is due to \cref{eq:near-efficiency-condition}. So we get $(1- \eta) I_q \preceq \widetilde{S} \preceq (1+\eta) I_q$,  which implies that for any eigenvalue $\lambda$ of $\widetilde{S}$, $\lambda \in [1- \eta, 1+ \eta]$ and is strictly positive. Therefore,
		\begin{align}
			(\lambda - (1- \eta)) ((1+\eta) - \lambda) \geq 0 \implies 
			\lambda + \frac{1-\eta^2}{\lambda} \leq 2 \implies 
			\widetilde{S} + (1-\eta^2) \widetilde{S}^{-1} \preceq 2 I_q. \label{eq:eta-bound-Stilde}
		\end{align}
		Define $\widetilde{G} = W^{\frac{1}{2}} G$. Substituting $G = W^{- \frac{1}{2}} \widetilde{G} $ and $S^{-1} = W^{\frac{1}{2}} \widetilde{S}^{-1} W^{\frac{1}{2}}$ into the asymptotic covariances $V^{(\dgmm)} = \left(G^\top W G\right)^{-1} G^\top WSWG \left(G^\top W G\right)^{-1}, V^{(\gmm)} = \left(G^\top S^{-1} G\right)^{-1}$:
		\begin{align}
			V^{(\dgmm)} = \left(\widetilde{G}^\top \widetilde{G} \right)^{-1} \widetilde{G}^\top \widetilde{S} \widetilde{G}  \left(\widetilde{G}^\top \widetilde{G} \right)^{-1}, \quad 			V^{(\gmm)}  = \left(\widetilde{G}^\top \widetilde{S}^{-1} \widetilde{G} \right)^{-1}. \label{eq:change-of-coord-variance}
		\end{align}
		
		By combining \cref{eq:eta-bound-Stilde} and \cref{eq:change-of-coord-variance} and using the shorthand $ H \coloneq \left(\widetilde{G}^\top \widetilde{G} \right)^{-1}  $, we get 
		\begin{align*}
			H \widetilde{G}^\top \left( \widetilde{S} + (1-\eta^2) \widetilde{S}^{-1}  \right)\widetilde{G}  	 H 
			&\preceq 		 H  \widetilde{G}^\top \left( 2 I_q\right)  \widetilde{G}  	 H \\
			V^{(\dgmm)}  + 	(1-\eta^2)   	 H  \left(V^{(\gmm)} \right)^{-1} 	 H  &\preceq 2 	 H \\
			V^{(\dgmm)}  & \preceq 2 	 H  - (1-\eta^2)   	 H  \left(V^{(\gmm)} \right)^{-1} 	 H . \numberthis \label{eq:upper-bound-step}
		\end{align*}
		
		By a completion-of-the-square argument inspired by the proof of \cite[Theorem 3.2]{hansen1982large}, we get 
		\begin{align*}
			&\frac{1}{1-\eta^2} V^{(\gmm)} - \left(  2 H - (1- \eta^2) H \left(V^{(\gmm)} \right)^{-1} H \right) \\
			=& (1- \eta^2) \left( H - \frac{1}{1-\eta^2} V^{(\gmm)} \right) \left(V^{(\gmm)} \right)^{-1} \left( H  - \frac{1}{1-\eta^2} V^{(\gmm)} \right) \succeq 0,
		\end{align*}
		where the last step is due to $\left(V^{(\gmm)} \right)^{-1}  \succ 0, 1-\eta^2 >0$. It then follows that 
		\begin{align}
			2 H - (1- \eta^2)H \left(V^{(\gmm)} \right)^{-1} H \preceq \frac{1}{1-\eta^2} V^{(\gmm)}.\label{eq:complete-the-square-step}
		\end{align}
		
		Combining \cref{eq:upper-bound-step} and  \cref{eq:complete-the-square-step}, we get $V^{(\dgmm)}  \preceq \frac{1}{1-\eta^2} V^{(\gmm)} = \left(1 + \frac{\eta^2}{1- \eta^2}\right) V^{(\gmm)}.$
		
	\end{proofEnd}
	
	\proofsketch{The lower bound is a known result in large-sample GMM theory. After applying a change of coordinates  $\widetilde{S} = W^{\frac{1}{2}}SW^{\frac{1}{2}}$ and $\widetilde{G} = W^{\frac{1}{2}} G$, we can then bound the spectrum of $\widetilde{S}$ in $[1-\eta,1+\eta]$ and relate the sandwich forms of $V^{(\dgmm)}$ and $V^{(\gmm)}$ using a complete-the-square argument inspired by the proof in \cite[Theorem 3.2]{hansen1982large}.}
	
	\begin{remark}
		As the cross-order correlations across sample moment conditions become weaker and as sample moment conditions of the same order share comparable sampling variances, $\eta \to 0$ and the efficiency gap $V^{(\dgmm)} -V^{(\gmm)}$ is reduced. When $S$ is exactly block diagonal and each within-order block is a scalar multiple of the identity, $S^{(k, k^\prime)}=0$ for $k\neq k'$ and $S^{(k, k)}=\sigma_k^2I_{d^k}$, the DGMM weights in \cref{eq:w_k-direct} evaluate to $w_k=\sigma_k^{-2}$ and the condition \cref{eq:near-efficiency-condition} holds for all $k = 1, \dots, L$ with $\eta = 0$. This is when the DGMM estimator is exactly efficient. 
	\end{remark}
	
	Having established the statistical properties of the DGMM estimator in general, we provide a computationally efficient algorithm to obtain ${\widehat{\btheta}}^{(\dgmm)}$ specializing to \cref{model:heteroscedastic}. The main idea is to compute the DGMM weights $\widehat{w}_k^{[t]}$ and the analytic gradients of the DGMM objective $Q^{[t]}_N(\btheta)$ in \cref{eq:d-gmom-multistep} from these quantities that can be obtained via implicit moment computation: 
	\begin{equation}
		\label{eq:alpha-beta-gamma}
		\begin{aligned}
			\phi^{(k)} (\btheta) \coloneq \norm{\cM^{(k)}\left({\btheta}\right)}^2, 
			\psi^{(k)} (\btheta; \by_{n}) \coloneq \inner{\cM^{(k)}({\btheta})}{\by_{n}^{\otimes k}},  C_{k,n,n^\prime} \coloneq \inner{{\by_n}^{\otimes k}}{\by_{n^\prime}^{\otimes k}},
		\end{aligned}
	\end{equation} 
	whose computational complexities are shown in \cref{thm:computational-complexity-norm,thm:computational-complexity-inner,thm:computational-complexity-inner-prod-sum}.  \cref{thm:computational-complexity-norm,thm:computational-complexity-inner} are specializations of \cite[Theorem~4.2 and Theorem~4.3]{pereira2022tensor} to the low-rank parameterization $\Sigma_j = V_j V_j^\top$ of \cref{model:heteroscedastic}. Theorems related to computing the inner product sum in  \cref{thm:computational-complexity-inner-prod-sum,lemma:approximation-error-kernel,lemma:rank-algebraic-bound} and the subsequent use of these primitives to derive the complexity of DGMM estimator in \cref{thm:computational-complexity-W,thm:computational-complexity-gradients,thm:computational-complexity-comparison} are new results of this work.
	
	\begin{theoremEnd}[end]{theorem}[Computational complexity of $\phi^{(k)} (\btheta)$ and its gradients]
		\label{thm:computational-complexity-norm}
		\begin{align}
			\phi^{(k)} (\btheta) = \sum_{i=1}^K \sum_{j=1}^K  \pi_i \pi_j \left\langle  \cM^{(k)}({\btheta}_i), \cM^{(k)}({\btheta}_j)\right\rangle = \sum_{i=1}^K \sum_{j=1}^K \pi_i \pi_j B_k\left(\kappa^{(1)}_{ij}  , \dots, \kappa^{(k)}_{ij}  \right)
		\end{align}
		and can be computed in $\cO(k^2 K^2 + k K^2 d R_{\max}^2)$ operations. Its gradients with respect to the parameters $\nabla_{\pi_j} \phi^{(k)} (\btheta), \nabla_{{\bmu}_j} \phi^{(k)} (\btheta), \nabla_{V_j} \phi^{(k)} (\btheta)$ can each be computed in $\cO(k^2 K + k K d R_{\max}^2)$.
	\end{theoremEnd}
	
	\begin{proofEnd}
		\label{pf:computational-complexity-norm}
		We compute $\phi^{(k)} (\btheta)$ by applying \cref{prop:mom-cum-bell}:
		\begin{align*}\label{eq:F_1}
			\phi^{(k)} (\btheta) &= \sum_{i=1}^K \sum_{j=1}^K \pi_i \pi_j \left\langle \cM^{(k)}({\btheta}_i), \cM^{(k)}({\btheta}_j)\right\rangle = \sum_{i=1}^K \sum_{j=1}^K \pi_i \pi_j B_k\left(\kappa^{(1)}_{ij}  , \dots, \kappa^{(k)}_{ij}  \right) \numberthis,
		\end{align*}
		where $\cM^{(k)}({\btheta}_j)$ denotes the $k$-th moment of the $j$-th Gaussian component, ${\btheta}_j = \left[\pi_j; {\bmu}_j ; V_j \right]^\top$ are the parameters of the $j$-th Gaussian component, and $\kappa^{(l)}_{ij}$ denotes the $l$-th cumulant of the inner product between the $i$-th Gaussian component ${\bX}_i$ and an independent draw of the $j$-th Gaussian component ${\bX}_j^\prime$, i.e., $\kappa^{(l)}_{ij} \coloneq \kappa^{(l)}\left(\inner{{\bX}_i}{{\bX}_j^\prime}\right), l = 1, \dots, k$. The Bell polynomials are evaluated by applying the recurrence relation (see, e.g., \cite[Theorem~11.2]{charalambides2018enumerative} or \cite[Eq.~7.12]{bell1934exponential}):
		\begin{align}
			\label{eq:alpha-recurrence}
			\begin{cases*}
				B_0 = 1, & (base case)\\
				B_{k}\left( \kappa^{(1)}_{ij}, \dots, \kappa^{(k)}_{ij} \right) = \sum_{l=0}^{k-1} {k-1 \choose l}B_{k-l-1}\left(\kappa^{(1)}_{ij}, \dots, \kappa^{(k-l-1)}_{ij}\right) \kappa^{(l+1)}_{ij}. & (induction step)
			\end{cases*}
		\end{align}
		By applying \cite[Proposition 3.3]{pereira2022tensor}, 
		\begin{align} 
			\label{eq:cumulants}
			\kappa^{(l)}_{ij} =
			\begin{cases*}
				\left\langle{\bmu}_j, {\bmu}_i \right\rangle, & ($l = 1$) \\
				(l-1)! \Tr((V_i {V_i}^\top V_j {V_j}^\top)^{\frac{l}{2}}) + \frac{l!}{2}  {{\bmu}_i}^\top V_j {V_j}^\top  (V_i {V_i}^\top V_j {V_j}^\top)^{\frac{l-2}{2}} {\bmu}_i \\
				\quad + \frac{l!}{2} {{\bmu}_j}^\top (V_i {V_i}^\top V_j {V_j}^\top)^{\frac{l-2}{2}}V_i {V_i}^\top {\bmu}_j, & ($l$ is even) \\
				l! {{\bmu}_j}^\top (V_i {V_i}^\top V_j {V_j}^\top)^{\frac{l-1}{2}} {\bmu}_i. & ($l$ is odd) \\
			\end{cases*}
		\end{align}
		Note that the computational complexity of the matrix multiplication $V_j^\top V_j$ is $\cO(d R_j^2)$, better than that of $V_j V_j^\top$, which is $\cO(d^2R_j)$. By carefully choosing the order of the consecutive matrix multiplication, we compute $\kappa^{(l)}_{ij}$ in \cref{eq:cumulants} using $\cO(ldR_{\max}^2)$ operations\footnote{Our implementation also follows this faster order.}. $\phi^{(k)} (\btheta)$ in \cref{eq:F_1} requires (1) pre-computing $\kappa^{(l)}_{ij}$ for all $i,j \in [K]$ and $l\in [k]$ in $\cO(k K^2 d R_{\max}^2)$ operations, (2) using the recurrence relation in \cref{eq:alpha-recurrence} to compute $B_{k}\left( \kappa^{(1)}_{ij}, \dots, \kappa^{(k)}_{ij} \right)$ for all $i,j \in [K]$ in $\cO(k^2 K^2)$ operations. In total, $\phi^{(k)} (\btheta)$ takes $\cO(k^2 K^2 + k K^2 d R_{\max}^2)$ operations.
		
		The derivative of the Bell polynomials is given by {\cite[Eq. 5.1]{bell1934exponential}}:
		\begin{equation}
			\frac{\partial B_k}{\partial x_i}(x_1, \dots, x_k) = \binom{k}{i} B_{k-i} (x_1, \dots, x_{k-i}).
		\end{equation}
		Using \cref{eq:F_1}, the above identity, and the chain rule, we compute:
		\begin{align*} 
			\nabla_{\pi_j} \phi^{(k)} (\btheta) &= 2 \sum_{i=1}^K \pi_i B_{k} \left(\kappa^{(1)}_{ij},\dots, 
			\kappa^{(k)}_{ij}\right),\numberthis\\
			\nabla_{{\bmu}_j}\phi^{(k)} (\btheta) &= 2 \sum_{i=1}^K   \pi_i \pi_j \sum_{l=1}^k \binom{k}{l} B_{k-l} \left(\kappa^{(1)}_{ij},\dots, 
			\kappa^{(k-l)}_{ij}\right) \nabla_{{\bmu}_j} \kappa^{(l)}_{ij},\numberthis\\
			\nabla_{V_j} \phi^{(k)} (\btheta) &= 2 \sum_{i=1}^K \pi_i \pi_j \sum_{l=1}^k \binom{k}{l} B_{k-l} \left(\kappa^{(1)}_{ij},\dots, 
			\kappa^{(k-l)}_{ij}\right) \nabla_{V_j} \kappa^{(l)}_{ij}.\numberthis
		\end{align*}
		From the cumulants in \cref{eq:cumulants}, we get the gradients of the cumulants w.r.t.~${\bmu}_j$ and $V_j$:
		\begin{equation}
			\label{eq:cumulants-gradients-mu}
			\nabla_{{\bmu}_j} \kappa^{(l)}_{ij} = 
			\begin{cases*}
				{\bmu}_i , & ($l = 1$) \\
				l! \left(V_i {V_i}^\top V_j {V_j}^\top\right)^{\frac{l-2}{2}} V_i {V_i}^\top {\bmu}_j, & ($l$ is even) \\
				l! \left(V_i {V_i}^\top V_j {V_j}^\top\right)^{\frac{l-1}{2}} {\bmu}_i, & ($l$ is odd) \\
			\end{cases*}
		\end{equation}
		\begin{align}
			\label{eq:cumulants-gradients-v}
			\nabla_{V_j} \kappa^{(l)}_{ij}  &= \begin{cases*}
				0 , & ($l = 1$) \\
				l! \left(V_i {V_i}^\top V_j {V_j}^\top\right)^{\frac{l-2}{2}} V_i {V_i}^\top V_j \\
				\quad + l! \sum_{p=0}^{\frac{l-2}{2}} \left(V_i {V_i}^\top V_j {V_j}^\top\right)^p {\bmu}_i {{\bmu}_i}^\top \left(V_j {V_j}^\top V_i {V_i}^\top\right)^{\frac{l-2}{2} - p} V_j \\
				\quad +  l! \sum_{p=0}^{\frac{l-2}{2}-1} V_i {V_i}^\top \left(V_j {V_j}^\top V_i {V_i}^\top\right)^p {\bmu}_j {{\bmu}_j}^\top \left(V_i {V_i}^\top V_j {V_j}^\top\right)^{\frac{l-2}{2}-1 - p} V_i {V_i}^\top V_j, & ($l$ is even) \\
				l! \sum_{p=0}^{\frac{l-1}{2}-1} V_i {V_i}^\top \left(V_j {V_j}^\top V_i {V_i}^\top\right)^p {\bmu}_j {{\bmu}_i}^\top \left(V_j {V_j}^\top V_i {V_i}^\top\right)^{\frac{l-1}{2}-1-p} V_j \\
				\quad + l!  \sum_{p=0}^{\frac{l-1}{2}-1}  \left(V_i {V_i}^\top V_j {V_j}^\top\right)^p {\bmu}_i {{\bmu}_j}^\top \left(V_i {V_i}^\top V_j {V_j}^\top\right)^{\frac{l-1}{2}-1-p} V_i {V_i}^\top V_j. & ($l$ is odd)
			\end{cases*}
		\end{align}
		$\nabla_{\pi_j}\phi^{(k)} (\btheta)$ requires (1) pre-computing $\kappa^{(l)}_{ij}$ for all $i \in [K]$, $l\in [k]$ in $\cO(k K d R_{\max}^2)$ operations, (2) using the recurrence relation in \cref{eq:alpha-recurrence} to compute $B_{k}\left( \kappa^{(1)}_{ij}, \dots, \kappa^{(k)}_{ij} \right)$ for all $i \in [K]$ in $\cO(k^2 K)$ operations. In total, $\nabla_{\pi_j} \phi^{(k)} (\btheta)$ takes $\cO(k^2 K + k K d R_{\max}^2)$ operations. $\nabla_{{\bmu}_j} \phi^{(k)} (\btheta)$ requires (1) pre-computing $\kappa^{(l)}_{ij}$ for all $i \in [K]$, $l\in [k]$ in $\cO(k K d R_{\max}^2)$ operations, (2) pre-computing $\nabla_{{\bmu}_j} \kappa^{(l)}_{ij}$ for all $i \in [K]$, $l\in [k]$ in $\cO(k K d R_{\max})$ operations, (3) using the recurrence relation in \cref{eq:alpha-recurrence} to compute $B_{k-l}\left(\kappa^{(1)}_{ij}, \dots, \kappa^{(k-l)}_{ij} \right), l=1, \dots,k$ for all $i \in [K]$ in $\cO(k^2 K)$ operations. In total, $\nabla_{{\bmu}_j} \phi^{(k)} (\btheta)$ takes $\cO(k^2 K + k K d R_{\max}^2)$ operations. With a similar calculation, $\nabla_{V_j}\phi^{(k)} (\btheta)$ takes $\cO(k^2 K + k K d R_{\max}^2)$ operations.
		\space\end{proofEnd}
	
	\begin{theoremEnd}[end]{theorem}[Computational complexity of $\psi^{(k)} (\btheta; \by_{n})$ and its gradients]
		\label{thm:computational-complexity-inner}
		\begin{align}
			\psi^{(k)} (\btheta; \by_{n}) =  \sum_{j=1}^K \pi_j \left\langle \cM^{(k)}({\btheta}_j), {\by_n}^{\otimes k}\right\rangle = \sum_{j=1}^K \pi_j B_{k}\left({{\by_n}}^\top {\bmu}_j, {{\by_n}}^\top V_j {V_j}^\top {\by_n}, 0, \dots, 0\right)
		\end{align}
		and can be computed in $\cO(k K +  K dR_{\max})$ operations. Its gradients $\nabla_{\pi_j} \psi^{(k)} (\btheta; \by_{n})$, $\nabla_{{\bmu}_j} \psi^{(k)} (\btheta; \by_{n})$, and $\nabla_{V_j} \psi^{(k)} (\btheta; \by_{n})$ can each be computed in $\cO(k K +  K dR_{\max})$ operations.
	\end{theoremEnd}
	\begin{proofEnd}
		\label{pf:computational-complexity-inner}
		We compute $\psi^{(k)} (\btheta; \by_{n})$ by applying \cref{prop:mom-cum-bell}:
		\begin{align}
			\psi^{(k)} (\btheta; \by_{n}) =  \sum_{j=1}^K  \pi_j  \left\langle \cM^{(k)}({\btheta}_j), {\by_n}^{\otimes k}\right\rangle = \sum_{j=1}^K \pi_j B_{k}\left({{\by_n}}^\top {\bmu}_j, {{\by_n}}^\top V_j {V_j}^\top {\by_n}, 0, \dots, 0\right).\label{eq:F_2}
		\end{align}
		The Bell polynomials are evaluated by applying the recurrence relation:
		\begin{align}
			\label{eq:beta-recurrence}
			\begin{cases*}
				B_0 = 1, & (base case)\\
				B_{1}\left({{\by_n}}^\top {\bmu}_j, {{\by_n}}^\top V_j {V_j}^\top {\by_n}, 0, \dots, 0\right) = {{\by_n}}^\top {\bmu}_j, & (base case) \\
				B_{k}\left({{\by_n}}^\top {\bmu}_j, {{\by_n}}^\top V_j {V_j}^\top {\by_n}, 0, \dots, 0\right) =  \\
				\quad B_{k-1}\left({{\by_n}}^\top {\bmu}_j, {{\by_n}}^\top V_j {V_j}^\top {\by_n}, 0, \dots, 0\right)  {{\by_n}}^\top {\bmu}_j \\
				\quad + (k-1) B_{k-2}\left({{\by_n}}^\top {\bmu}_j, {{\by_n}}^\top V_j {V_j}^\top {\by_n}, 0, \dots, 0\right) {{\by_n}}^\top V_j {V_j}^\top {\by_n}. & (induction step)
			\end{cases*}
		\end{align}
		Using \cref{eq:F_2}, the chain rule, and the derivative of the Bell polynomials, we compute $\nabla_{\pi_j} \psi^{(k)} (\btheta; \by_{n})$, $\nabla_{{\bmu}_j} \psi^{(k)} (\btheta; \by_{n}), \nabla_{V_j} \psi^{(k)} (\btheta; \by_{n})$ for each $j$-th component:
		\begin{align}
			\nabla_{\pi_j}\psi^{(k)} (\btheta; \by_{n}) &= B_{k}\left({{\by_n}}^\top {\bmu}_j, {{\by_n}}^\top V_j {V_j}^\top {\by_n}, 0, \dots, 0\right), \\
			\nabla_{{\bmu}_j} \psi^{(k)} (\btheta; \by_{n}) &= k \pi_j  B_{k-1}\left({{\by_n}}^\top {\bmu}_j, {{\by_n}}^\top V_j {V_j}^\top {\by_n}, 0, \dots, 0\right) {\by_n},  \\
			\nabla_{V_j} \psi^{(k)} (\btheta; \by_{n}) &= 			\begin{cases*} 
				0, & ($k = 1$)\\
				2 \binom{k}{2} \pi_j B_{k-2}\left({{\by_n}}^\top {\bmu}_j, {{\by_n}}^\top V_j {V_j}^\top {\by_n}, 0, \dots, 0\right){\by_n} {{\by_n}}^\top V_j, & ($k \geq 2)$.
			\end{cases*}
		\end{align}
		For all $j\in [K]$, ${{\by_n}}^\top {\bmu}_j$, and ${{\by_n}}^\top V_j {V_j}^\top {\by_n}; \by_n {{\by_n}}^\top V_j$ can be pre-computed in $\cO(K d)$, $\cO(K dR_{\max})$, and $\cO(K dR_{\max})$, respectively, by carefully choosing the order of matrix multiplication. For all $j \in [K]$, using the recurrence relation in \cref{eq:beta-recurrence}, we can compute $B_{k}\left({{\by_n}}^\top {\bmu}_j, {{\by_n}}^\top V_j {V_j}^\top {\by_n}, 0, \dots, 0\right)$  in $\cO(k K)$ operations. In total, $\psi^{(k)} (\btheta; \by_{n})$, $\nabla_{\pi_j} \psi^{(k)} (\btheta; \by_{n})$,  $\nabla_{{\bmu}_j} \psi^{(k)} (\btheta; \by_{n})$, and $\nabla_{V_j} \psi^{(k)} (\btheta; \by_{n})$ each requires $\cO(k K +  K dR_{\max})$ operations. 
		\space\end{proofEnd}
	
	The next three results concern the computation of the polynomial kernel row sums $\sum_{n^\prime} C_{k,n,n^\prime}$ that are used to compute the DGMM weights. In \cref{thm:computational-complexity-inner-prod-sum}, we derive the complexity of approximately computing the row sums via randomized Nystr\"om approximation, with the detailed computational steps given in the proof. In \cref{lemma:approximation-error-kernel}, we apply the kernel $k$-means++ guarantee of \cite{oglic2017nystrom} to derive an expected error bound for the kernel approximation. Finally,  \cref{lemma:rank-algebraic-bound} is an algebraic  bound on the kernel rank under \cref{model:heteroscedastic}.
	
	\begin{theoremEnd}[end]{theorem}[Computational complexity of $\sum_{n^\prime} C_{k,n,n^\prime}$]
		\label{thm:computational-complexity-inner-prod-sum}
		Define the polynomial kernel $h^{(k)}(\by_n,\by_{n^\prime}) =  \langle \by_n,\by_{n^\prime}\rangle^{k}$ and its associated Gram matrix $H^{(k)}\in\R^{N\times N}$, where $H^{(k)}_{n,n^\prime} = \inner{\by_n}{\by_{n^\prime}}^k$ for $\by_n, \by_{n^\prime} \iid$  \cref{model:heteroscedastic}. Using Nystr\"om approximation\footnote{Standard treatments of Nyström-based kernel approximation can be found in \cite{williams2000using, drineas2005nystrom}.} with $m$ landmarks from kernel $k$-means++ sampling as in \cite{oglic2017nystrom}, $\sum_{n^\prime}C_{k,n,n^\prime}$ can be approximated in $\cO(Ndm)$ operations, with the expected error bound for the Nystr\"om approximation given in \cref{lemma:approximation-error-kernel}. 
	\end{theoremEnd}
	
	\begin{proofEnd}
		\label{pf:computational-complexity-inner-prod-sum}
		The sum $\sum_{n^\prime} C_{k,n,n^\prime}$ can be approximated in the following steps: 
		\begin{enumerate}
			\item Select $m$ landmarks from $N$ points using $k$-means++ sampling proposed in \cite{oglic2017nystrom}, which takes $\cO(Nd m)$ operations ($d$-dimensional distances for $N$ points for $m$ centroids).
			\item Compute $C^{(k)}$, an $N$-by-$m$ matrix containing kernel evaluations between all $N$ data points and the $m$ landmarks, which takes $\cO(N d m)$ operations. 
			\item Compute $W^{(k)}$, an $m$-by-$m$ Gram matrix on the $m$ landmarks for the $k$-th order moment, which takes $\cO(m^2 d)$ operations.
			\item  Apply pivoted Cholesky decomposition to factorize $W^{(k)} \approx L^{(k)} \left(L^{(k)}\right)^\top$, where $L^{(k)}$ is the Cholesky factor of rank $R_{\text{Chol}} < m$. This takes $\cO\left(m R_{\text{Chol}}^2\right)$ operations.
			\item To compute the inverse $\left(W^{(k)}\right)^{-1} \bv$ (which is required for the Nyström approximation), we perform a forward and back substitution through the triangular factor: $L^{(k)} \bz = \bv$ (forward substitution) and then $\left(L^{(k)}\right)^\top \by = \bz$ (back substitution). This step takes $\cO(m^2)$ operations.
			\item After obtaining the solution vector $\by$ from the triangular solve step, we compute the approximate $\sum_{n^\prime}C_{k,n,n^\prime} \approx C^{(k)} \by$, which takes $\cO(N m)$ operations.
			\item[] Thus, the overall computational complexity is $\cO\left(Ndm + m^2 d + mR_{\text{Chol}}^2  + m^2 + Nm \right)$. Taking the dominant terms, we get $\cO\left(Ndm\right)$.
		\end{enumerate}
	\end{proofEnd}
	
	\proofsketch{Since we only require an approximation of the row sums, neither the $N\times N$ kernel matrix nor the kernel approximation is formed. The dominant computation is landmark selection and evaluations between $N$ observations and $m$ landmarks, which cost $\cO(Ndm)$.}
	
	\begin{theoremEnd}[end]{lemma}[Nystr\"om approximation expected error bound]
		\label{lemma:approximation-error-kernel}
		Let ${H^{(k)}_m}$ be the optimal rank-$m$ approximation of $H^{(k)}$ and $\widetilde{H^{(k)}_m}$ be the Nystr\"om approximation of $H^{(k)}$ using $m$ landmarks from kernel $k$-means++ sampling as in \cite{oglic2017nystrom}. Denote by $\lambda^{(k)}_{i}$ the $i$‑th eigenvalue of $H^{(k)}$. Suppose that $H^{(k)}$ exhibits polynomial spectral decay,\footnote{This is a typical assumption for Nystr\"om approximation. See, e.g., \cite[Section 3]{oglic2017nystrom} and \cite[Section 4.3]{bach2013sharp}.} i.e., $\lambda_i^{(k)} \in \cO(i^{-a})$ with $a>1$ and that the clustering potential condition of \cite{oglic2017nystrom} holds, i.e.,
		\begin{align}
			\label{ineq:cluster-aligned-condition}
			\phi\left(C^\star \mid U_{m}\right) \leq \sqrt{N-m}\norm{H^{(k)}- H^{(k)}_m}_F,
		\end{align}
		where $C^\star$ is an optimal $(m+1)$-clustering, $U_{m}$ is the span of the top $m$ eigenvectors of $H^{(k)}$, and $\phi\left(C^\star \mid U_{m}\right)$ is the clustering potential (\cite{ding2004k, boutsidis2009unsupervised}) given by the projections of $C^\star$ onto the subspace $U_{m}$. Then we have the following expected error bound:
		\begin{align}
			\label{eq:nystrom-expected-bound}
			\EX\left[\norm{H^{(k)}-\widetilde{H^{(k)}_m}}_F\right]\in \cO\left(\ln(m) \sqrt{N - m}\left(m^{-a + 1/2}\right) \right),
		\end{align}
		where the expectation is taken over random choice of landmarks. 
		In particular, $\EX\left[\norm{H^{(k)}-\widetilde{H^{(k)}_m}}_F\right] \to 0$ as $m \to \rank H^{(k)}$, which has the algebraic bound in \cref{lemma:rank-algebraic-bound}.
	\end{theoremEnd}
	
	\begin{proofEnd}
		\label{pf:approximation-error-kernel}
		Applying \cite[Theorem~5]{oglic2017nystrom} to our problem, we get
		\begin{align}
			\label{ineq:fraction-error-bound}
			\EX\left[\frac{\norm{H^{(k)} - \widetilde{H^{(k)}_m}}_F}{\norm{H^{(k)} - H^{(k)}_m}_F}\right] &\leq 8\left(\ln(m+1)+2\right) \left(\sqrt{N - m}+ \frac{\phi(C^\star\mid U_{m})}{ \norm{H^{(k)}- H^{(k)}_m}_F}\right).
		\end{align}
		
		Under polynomial spectral decay of $H^{(k)}$, $\lambda_i^{(k)} \in \cO(i^{-a})$ with $a>1$, i.e., there exists some constant $C > 0$ such that for large $i$, $\lambda_i^{(k)} \leq C i^{-a}$ for $a>1$. By the integral test, we have
		\begin{align*}
			\norm{H^{(k)}-H^{(k)}_{m}}_F^2 = \sum_{i= m+1}^{N}\left(\lambda_{i}^{(k)}\right)^2 &\leq \sum_{i= m+1}^{N}\left(C i^{-a}\right)^2 \leq C^2 \int_{m}^{N} x^{-2a} \mathrm{d}x\\
			&= \frac{C^2}{2a-1}\left(m^{-2a+1} - N^{-2a+1}\right) \leq \frac{C^2}{2a-1}\left(m^{-2a+1}\right).
		\end{align*}
		
		After taking the square root on both sides, we get $\norm{H^{(k)}-H^{(k)}_{m}}_F \leq \frac{C}{\sqrt{2a-1}}\left(m^{-a+1/2}\right)$. Combining this, \cref{ineq:fraction-error-bound}, and the cluster potential condition in \cref{ineq:cluster-aligned-condition}, we get
		\begin{align*}
			\EX\left[\norm{H^{(k)} - \widetilde{H^{(k)}_m}}_F\right] &\leq 8\left(\ln(m+1)+2\right) \left(\sqrt{N - m} \norm{H^{(k)}- H^{(k)}_m}_F + \phi(C^\star\!\mid U_{m})\right)\\
			&\leq 8 \left(\ln(m+1)+2\right) \left(2\sqrt{N - m} \norm{H^{(k)}- H^{(k)}_m}_F\right)\\
			&\leq 16 \left(\ln(m+1)+2\right) \sqrt{N - m} \frac{C}{\sqrt{2a-1}}\left(m^{-a+1/2}\right),\\
			\text{which implies that } & \EX\left[\norm{H^{(k)}-\widetilde{H^{(k)}_{m}}}_{F}\right] \in \cO\left(\ln(m) \sqrt{N- m}\left(m^{-a+1/2}\right)\right). \text{ In particular, } \\
			\text{as } m \to \rank H^{(k)}, \text{ }& \text{the expected error goes to zero.}
		\end{align*}
	\end{proofEnd}
	\proofsketch{By applying \cite[Theorem~5]{oglic2017nystrom} and bounding $\norm{H^{(k)}-H_m^{(k)}}_F$ by the eigenvalue decay via an integral test, we get \cref{eq:nystrom-expected-bound}.}
	
	\begin{theoremEnd}[end]{lemma}[Algebraic bound on $\rank H^{(k)}$]\label{lemma:rank-algebraic-bound}  For observations of \cref{model:heteroscedastic}, $\rank H^{(k)}\leq \sum_{j=1}^{K}\binom{R_{j}+k}{k} \leq K \binom{R_{\max}+k}{k}$. If all mixture component supports $\bmu_j + \range(\Sigma_j)$ are contained in a common linear subspace $U$ with $\dim U \leq R_{\max}$, then $\rank H^{(k)}\leq \binom{R_{\max}+k-1}{k}.$
	\end{theoremEnd}
	
	\begin{proofEnd}
		\label{pf:rank-algebraic-bound}
		Since $\by_{1},\dots,\by_{N} \iid$ \cref{model:heteroscedastic}, by the characterization of its distribution given in \cref{prop:pdf-single-gaussian-lowrank}, $\by_{1},\dots,\by_{N}$ lie in a union of $K$ affine subspaces, $\bigcup_{j=1}^{K} U_{j}, \quad \dim U_{j}=R_{j}\leq R_{\max}.$
		Every point $\by \in U_{j} = \bmu_j + \range(\Sigma_j)$ lies in the linear subspace $\text{span}\{\bmu_j\} + \range(\Sigma_j) $, whose  dimension is at most $R_j + 1$. 
		Writing $\by$ in a basis of this subspace, every coordinate of $\by^{\otimes k}$ is then a homogeneous polynomial of degree $k$ in at most $R_j + 1$ variables. Since the number of distinct monomials of degree $k$ in $R_j + 1$ variables is $\binom{R_j+k}{k}$, $\by^{\otimes k}$ with $\by \in U_j$ span at most $\binom{R_j+k}{k}$ directions.
		
		The global span is therefore contained in the sum of these $K$ subspaces, with dimension bounded by $\sum_{j=1}^{K}\binom{R_{j}+k}{k}$. Then, $\rank H^{(k)} \leq \sum_{j=1}^{K}\binom{R_{j}+k}{k} \leq K \binom{R_{\max}+k}{k}$. If we further assume that $U_j$ is contained in a single linear subspace $U$ with $\dim U \leq R_{\max}$ for all $j$, then all observations lie in the linear subspace $U$ itself. In this case, we would not need the additional dimension for $\bmu_j$, and $\by^{\otimes k}$ span at most $\binom{\dim U+k-1}{k} \leq \binom{R_{\max}+k-1}{k}$ directions.
	\end{proofEnd}
	\proofsketch{On the $j$-th component support, $\by \in U_{j} \subseteq \text{span}\{\bmu_j\} + \range(\Sigma_j)$, whose dimension is at most $R_j + 1$. Therefore, $\by^{\otimes k}$ span at most $\binom{R_j+k}{k}$ directions, the number of distinct monomials of degree $k$ in $R_j + 1$ variables. If all observations lie in a common linear subspace $U$ of dimension at most $R_{\max}$, the bound becomes $\rank H^{(k)}\leq \binom{R_{\max}+k-1}{k}.$}
	
	\begin{remark}
		To interpret the bounds in \cref{lemma:approximation-error-kernel,lemma:rank-algebraic-bound} in practice, the spectral decay of $H^{(k)}$ and the cluster structure in the data should be considered. 
		\begin{itemize}
			\item A faster spectral decay leads to a faster expected error decay, which implies that a smaller number of landmarks $m$ is required to achieve the same approximation error. For i.i.d.~data from \cref{model:heteroscedastic}, we observe that $H^{(k)}$ exhibits fast spectral decay. Consequently, the actual number of landmarks required is considerably smaller than the algebraic bound in \cref{lemma:rank-algebraic-bound}. However, for data whose corresponding $H^{(k)}$ does not exhibit spectral decay, the number of landmarks required will be closer to the algebraic bound in \cref{lemma:rank-algebraic-bound}. 
			\item The clustering potential $\phi\left(C^\star \mid U_{m}\right)$ in \cref{ineq:cluster-aligned-condition} is the sum of within-cluster variances under the optimal cluster assignment after projecting onto $U_m$, which measures how ``clusterable'' the data is. For i.i.d.~data from \cref{model:heteroscedastic}, the inherent cluster structure means that the points are still concentrated around their respective centroids after projecting onto $U_m$, leading to small clustering potential. However, for data without a cluster structure, the clustering potential condition \cref{ineq:cluster-aligned-condition} might not hold.
		\end{itemize}
	\end{remark}
	
	Combining the preceding results, we obtain \cref{thm:computational-complexity-W,thm:computational-complexity-gradients,thm:computational-complexity-comparison}. These are original results of the present work, where we derive the complexity of DGMM weights and gradients, as well as the overall complexity of the DGMM estimator in comparison to standard MM and GMM.
	
	\begin{theoremEnd}[end]{theorem}[Computational complexity of $\widehat{w}_k^{[t]}$]
		\label{thm:computational-complexity-W}
		The proposed weights $\widehat{w}_k^{[t]}$ in \cref{eq:d-gmom-multistep} can be expressed in terms of the quantities in \cref{eq:alpha-beta-gamma}:
		\begin{align*}
			\label{eq:w_k}
			\widehat{w}^{[t]}_k 
			&= \frac{ N \sum_{n} \left(\phi_k^{[t-1]} - 2 \psi_{k,n}^{[t-1]} +C_{k, n, n}\right) }{\sum_{n, n^\prime} \sum_{k^\prime}\left(\phi_k^{[t-1]} -  \psi_{k,n}^{[t-1]} - \psi_{k,n^\prime}^{[t-1]} + C_{k,n,n^\prime}\right)\left(\phi_{k^\prime}^{[t-1]} -  \psi_{k^\prime,n}^{[t-1]} - \psi_{k^\prime,n^\prime}^{[t-1]} + C_{k^\prime, n, n^\prime}\right)},\numberthis
		\end{align*}
		where $	\phi_k^{[t-1]} \coloneq \phi^{(k)}\left(\widehat{\btheta}^{[t-1]}\right),
		\psi_{k,n}^{[t-1]} \coloneq \psi^{(k)}\left(\widehat{\btheta}^{[t-1]}; \by_n\right).$		Further, $ \widehat{w}_k^{[t]}$ can be computed in $\cO(L^2 K^2 + L K^2 d R_{\max}^2 + NLK + NKdR_{\max})$ operations.
	\end{theoremEnd}
	
	\begin{proofEnd}
		\label{pf:computational-complexity-W}
		The numerator of $\widehat{w}_k^{[t]}$ in \cref{eq:d-gmom-multistep} can be written as
		\begin{equation*}
			\begin{split}
				&\sum_{i \in \cI_k} \widehat{S}^{[t]}_{ii} = \sum_{i \in \cI_k} \left(\frac{1}{N}\sum_{n=1}^{N} g(\widehat{{\btheta}}^{[t-1]}; \by_n) g(\widehat{{\btheta}}^{[t-1]}; \by_n)^\top \right)_{ii} = \sum_{i \in \cI_k}\frac{1}{N} \sum_{n=1}^{N} g(\widehat{{\btheta}}^{[t-1]}; \by_n)_i^2\\
				&= \frac{1}{N} \sum_{n=1}^{N}  \norm{\vectorize{\left(\cM^{(k)}\left(\widehat{{\btheta}}^{[t-1]}\right) - {\by_n}^{\otimes k}\right)}}^2 = \frac{1}{N} \sum_{n=1}^{N} \norm{\cM^{(k)}\left(\widehat{{\btheta}}^{[t-1]}\right) - {\by_n}^{\otimes k}}^2\\
				&= \frac{1}{N} \sum_{n=1}^{N} \norm{\cM^{(k)}\left(\widehat{{\btheta}}^{[t-1]}\right)}^2 - 2 \inner{\cM^{(k)}\left(\widehat{{\btheta}}^{[t-1]}\right)}{{\by_n}^{\otimes k}} +\inner{{\by_n}}{{\by_n}}^k.
			\end{split}
		\end{equation*}
		The denominator of $\widehat{w}_k^{[t]}$ in \cref{eq:d-gmom-multistep} can be written as
		\begin{equation*}
			\begin{split}
				&\sum_{i \in \cI_k}\sum_{j=1}^{q} \left(\widehat{S}^{[t]}_{ij}\right)^2 =  \sum_{i \in \cI_k}\sum_{j=1}^{q} \left(\frac{1}{N}\sum_{n=1}^{N} g(\widehat{{\btheta}}^{[t-1]}; \by_n) g(\widehat{{\btheta}}^{[t-1]}; \by_n)^\top \right)_{ij}^2\\
				&=   \frac{1}{N^2}\sum_{n,n^\prime =1}^{N} \sum_{i \in \cI_k} \left(\sum_{j=1}^{q}\left(g(\widehat{{\btheta}}^{[t-1]}; \by_n) g(\widehat{{\btheta}}^{[t-1]}; \by_n)^\top \right)_{ij} \left(g(\widehat{{\btheta}}^{[t-1]}; \by_{n^\prime}) g(\widehat{{\btheta}}^{[t-1]},\by_{n^\prime})^\top \right)_{ij} \right)\\
				=  &\frac{1}{N^2}\sum_{n,n^\prime =1}^{N} \inner{g(\widehat{{\btheta}}^{[t-1]}; \by_n)[i_k:{i_k}^{\prime}]}{g(\widehat{{\btheta}}^{[t-1]}; \by_{n^\prime})[i_k:{i_k}^{\prime}]} \Tr(g(\widehat{{\btheta}}^{[t-1]}; \by_n)g(\widehat{{\btheta}}^{[t-1]}; \by_{n^\prime})^\top)\\
				=  &\frac{1}{N^2}\sum_{n,n^\prime =1}^{N} \inner{g(\widehat{{\btheta}}^{[t-1]}; \by_n)[i_k:{i_k}^{\prime}]}{g(\widehat{{\btheta}}^{[t-1]}; \by_{n^\prime})[i_k:{i_k}^{\prime}]} \Tr(g(\widehat{{\btheta}}^{[t-1]}; \by_{n^\prime})^\top g(\widehat{{\btheta}}^{[t-1]}; \by_n))\\
				=  &\frac{1}{N^2}\sum_{n,n^\prime =1}^{N}  \inner{g(\widehat{{\btheta}}^{[t-1]}; \by_n)[i_k:{i_k}^{\prime}]}{g(\widehat{{\btheta}}^{[t-1]}; \by_{n^\prime})[i_k:{i_k}^{\prime}]} \inner{g(\widehat{{\btheta}}^{[t-1]}; \by_n))}{g(\widehat{{\btheta}}^{[t-1]}; \by_{n^\prime})}\\
				=  &\frac{1}{N^2}\sum_{n,n^\prime =1}^{N}  \inner{\cM^{(k)}\left(\widehat{{\btheta}}^{[t-1]}\right) - {\by_n}^{\otimes k }}{\cM^{(k)} - \by_{n^\prime}^{\otimes k }} \sum_{k^\prime=1}^L \inner{\cM^{(k^\prime)}\left(\widehat{{\btheta}}^{[t-1]}\right) - {\by_n}^{\otimes k^\prime}}{\cM^{(k^\prime)} - \by_{n^\prime}^{\otimes k^\prime}}
			\end{split}
		\end{equation*}
		Combining the numerator and the denominator, substituting with the quantities in \cref{eq:alpha-beta-gamma}, and simplifying, we get the expression for DGMM weights in \cref{eq:w_k}. Observe that $\widehat{w}^{[t]}_k$ requires (1)
		$\phi_k^{[t-1]} = \norm{\cM^{(k)}\left(\widehat{{\btheta}}^{[t-1]}\right)}^2, k = 1, \dots, L,$ which, by \cref{thm:computational-complexity-norm}, takes $\cO(L^2 K^2 + L K^2 d R_{\max}^2)$ operations, and (2) $\psi_{k,n}^{[t-1]} = \inner{\cM^{(k)}\left(\widehat{{\btheta}}^{[t-1]}\right)}{\by_{n}^{\otimes k}}$, $k = 1, \dots, L, n = 1, \dots, N,$ which, by \cref{thm:computational-complexity-inner}, takes $\cO(NLK + NKdR_{\max})$ operations. In total, $\widehat{w}^{[t]}_k$ in \cref{eq:d-gmom-multistep} can be computed in $\cO(L^2 K^2 + L K^2 d R_{\max}^2 + NLK + NKdR_{\max})$.
		\space\end{proofEnd}
	
	\begin{theoremEnd}[end]{theorem}[Computational complexity of $\nabla_{{\btheta}} Q^{[t]}_N(\btheta) $]
		\label{thm:computational-complexity-gradients}
		At each iteration of the first-order solver, the gradients $\nabla_{{\btheta}} Q^{[t]}_N(\btheta) $ can be expressed in terms of the quantities in \cref{eq:alpha-beta-gamma}:
		
		\begin{equation}
			\label{eq:gradients-efficient}
			\begin{aligned}
				\nabla_{{\btheta}} Q^{[t]}_N(\btheta)&  = \sum_{k=1}^L \widehat{w}^{[t]}_k  \left[\nabla_{\btheta} \phi^{(k)} (\btheta) -\frac{2}{N}\sum_{n=1}^N \nabla_{\btheta} \psi^{(k)} (\btheta; \by_{n})\right].
			\end{aligned}
		\end{equation}
		Further, $\nabla_{{\btheta}} Q^{[t]}_N(\btheta) $ can be computed in $\cO(L^2 K^2 + L K^2 dR_{\max}^2 + NLK + NKdR_{\max})$.
	\end{theoremEnd}
	
	\begin{proofEnd}
		\label{pf:computational-complexity-gradients}
		We rewrite the DGMM objective function in \cref{eq:d-gmom-multistep}:
		\begin{align*}
			Q^{[t]}_N(\btheta)  &= \overline{g}_N\left({\btheta}\right)^\top \diag(\underbrace{\widehat{w}^{[t]}_1, \dots, \widehat{w}^{[t]}_1}_{d \text{ copies}}, \dots, \underbrace{\widehat{w}^{[t]}_L, \dots, \widehat{w}^{[t]}_L}_{d^L \text{ copies}})    \overline{g}_N\left({\btheta}\right)\\
			=&\sum_{k=1}^L \widehat{w}^{[t]}_k \sum_{i \in \cI_k} \left(\vectorize\left(\cM^{(k)}\left({\btheta}\right)- \frac{1}{N} \sum_{n=1}^N {\by_n}^{\otimes k }\right)_i\right)^2=\sum_{k=1}^L \widehat{w}^{[t]}_k \norm{\cM^{(k)}\left({\btheta}\right) - \frac{1}{N} \sum_{n=1}^N {\by_n}^{\otimes k }}^2\\
			=&\sum_{k=1}^L \widehat{w}^{[t]}_k \left(\norm{\cM^{(k)}\left({\btheta}\right)}^2 - \frac{2}{N} \sum_{n=1}^N\left\langle \cM^{(k)} \left({\btheta}\right), {\by_n}^{\otimes k}\right\rangle + \frac{1}{N^2}\sum_{n=1}^N \sum_{n^{\prime}=1}^N \left\langle {\by_n}^{\otimes k},  \by_{n^{\prime}}^{\otimes k}\right\rangle\right),\\
			=&\sum_{k=1}^L \widehat{w}^{[t]}_k  \left(\phi^{(k)} (\btheta) - \frac{2}{N} \sum_{n=1}^N \psi^{(k)} (\btheta; \by_{n}) + \frac{1}{N^2}\sum_{n=1}^N \sum_{n^{\prime}=1}^N C_{k,n,n^\prime}\right).\numberthis\label{eq:gmm-objective-rewrite}
		\end{align*} 
		Observe that $\nabla_{{\btheta}} Q^{[t]}_N(\btheta)$ requires the following gradients: $\nabla_{\pi_j}\phi^{(k)} (\btheta), \nabla_{{\bmu}_j}\phi^{(k)} (\btheta),  \mkern9mu  \nabla_{V_j} \phi^{(k)} (\btheta)$ and $\nabla_{\pi_j} \psi^{(k)} (\btheta; \by_{n}), \nabla_{{\bmu}_j} \psi^{(k)} (\btheta; \by_{n}),  \nabla_{V_j} \psi^{(k)} (\btheta; \by_{n})$, for $k = 1, \dots, L, n = 1, \dots, N$, which, by \cref{thm:computational-complexity-norm} and \cref{thm:computational-complexity-inner}, takes $\cO(L^2 K^2 + LK^2 dR_{\max}^2 + NLK + NKdR_{\max})$ operations in total. $\nabla_{{\btheta}} Q^{[t]}_N(\btheta)$ is then obtained by vectorizing and concatenating these gradients.
		
		\space\end{proofEnd}
	
	Finally, we compare the computational complexity of DGMM with MM and GMM:
	
	\begin{theoremEnd}[end]{theorem}[Overall computational complexity of DGMM]
		\label{thm:computational-complexity-comparison}
		The DGMM estimator in \cref{eq:d-gmom-multistep} is obtained in $\cO(L^2 K^2 + L K^2 dR_{\max}^2 + NLK + NKdR_{\max} + NLdm)$ operations. In contrast, the MM estimator is obtained in $\cO(Nd^{L+1}KR_{\max})$ operations, and the GMM estimator is obtained in $\cO(Nd^{2L} + d^{3L} + Nd^{L+1}KR_{\max} + d^{2L+1}KR_{\max})$ operations.
		
	\end{theoremEnd}
	
	\begin{proofEnd}
		\label{pf:computational-complexity-comparison}
		Let $I$ be the number of L-BFGS iterations, $T$ the number of GMM steps. The DGMM estimator, $\widehat{\btheta}^{(\dgmm)}$, requires the following computations:
		\begin{enumerate}
			\item For one time, pre-compute $\sum_{n^\prime = 1}^N C_{k,n,n^\prime}, \text{ for } k = 1, \dots, 2L, n = 1, \dots, N,$ and store in a $(2L+1)\times N$ matrix. By \cref{thm:computational-complexity-inner-prod-sum}, this requires $\cO(NLdm)$ operations.
			\item For at most $T$ times, compute $\widehat{w}_k^{[t-1]}$ for $k = 1, \dots, L$, which, by \cref{thm:computational-complexity-W}, takes $\cO(T L^2 K^2 + T L K^2 d R_{\max}^2 + T NLK + T NKdR_{\max})$ operations. 
			\item For at most $I$ times, compute the necessary gradients, which, by \cref{thm:computational-complexity-gradients}, takes $\cO(I L^2 K^2 + I L K^2 dR_{\max}^2 + I NLK + I NKdR_{\max})$ operations. 
		\end{enumerate}
		Taking the dominant terms and absorbing the stopping criteria constants, we get that the overall computational complexity of the DGMM estimator ${{\widehat{\btheta}}^{(\dgmm)}}$ in \cref{eq:d-gmom-multistep} is $\cO(L^2 K^2 + L K^2 dR_{\max}^2 + NLK + NKdR_{\max} + NLdm)$.
		
		In contrast, the MM estimator, ${\widehat{\btheta}}^{(\mm)}$, requires the following computations:
		\begin{enumerate}
			\item For one time, pre-compute the sample moments, which takes $\cO(Nq)$, where $q = d + d^2 + \cdots + d^L$. Taking the dominant terms, we get $\cO(N d^{L})$.
			\item The weighting matrix is the identity matrix and therefore takes $\cO(1)$ time to compute.
			\item For at most $I$ times, compute the moment condition vector $\overline{g}_N(\btheta)$ in $\cO(Nq)$ operations, the objective function $\overline{g}_N(\btheta)^\top \overline{g}_N(\btheta)$ in $\cO(q^2)$ operations, the Jacobian matrix $G(\btheta) = \frac{\partial \overline{g}_N(\btheta)}{\partial \btheta} \in \R^{q\times p}$ in $\cO(Nqp)$ operations, and the gradients $\nabla_{\btheta} \overline{g}_N(\btheta)^\top\overline{g}_N(\btheta) = 2 G(\btheta)^\top \overline{g}_N(\btheta)$ in $\cO(qp)$ operations. Taking the dominant term and absorbing the stopping criteria constants, we get $\cO(INd^{L+1}KR_{\max}+ Id^{L+1}K R_{\max}) = \cO(Nd^{L+1}KR_{\max})$.
		\end{enumerate}
		
		Similarly, the GMM estimator, ${\widehat{\btheta}}^{(\gmm)}$, requires the following computations:
		\begin{enumerate}
			\item For one time, pre-compute the sample moments, which takes $\cO(N q)$.
			\item For at most $T$ times, compute $\widehat{S}^{[t]}$ in $\cO(N q^2)$ and invert $\widehat{S}^{[t]}$ in $\cO(q^3)$.
			\item For at most $I$ times, compute the moment condition vector $\overline{g}_N(\btheta)$ in $\cO(Nq)$ operations, the objective function $\overline{g}_N(\btheta)^\top \widehat{W}^{[t]}\overline{g}_N(\btheta)$ in $\cO(q^2)$ operations, the Jacobian matrix $G(\btheta) = \frac{\partial \overline{g}_N(\btheta)}{\partial \btheta}$ in $\cO(Nqp)$ operations, and the gradients $\nabla_{\btheta} \overline{g}_N(\btheta)^\top \widehat{W}^{[t]} \overline{g}_N(\btheta) = 2 G(\btheta)^\top \widehat{W}^{[t]} \overline{g}_N(\btheta)$ in $\cO(q^2 p)$ operations. Taking the dominant term, we get $\cO(INd^{L+1}KR_{\max}+ Id^{2L+1}K R_{\max})$.
			\item 
			The overall computational complexity of obtaining ${{\widehat{\btheta}}^{(\gmm)}}$ is $\cO(Nd^{2L} + d^{3L} + Nd^{L+1}KR_{\max}+ d^{2L+1}K R_{\max})$ after absorbing the stopping criteria constants.
		\end{enumerate}
		
	\end{proofEnd}
	
	\section{Numerical studies}
	\label{sec:numerical-studies}
	
	We test the DGMM procedure in \cref{eq:d-gmom-multistep} on parameter estimation for low-SNR heteroscedastic low-rank mixtures in \cref{model:heteroscedastic}, given the number of components $K$, the maximum rank $R_{\max}$, and $\{\by_n\}_{n=1}^N \in \R^d \iid$ \cref{model:heteroscedastic}. We consider two cases in \cref{tab:datasets}: (1) when the covariances of the mixture components share identical rank, i.e., $R_i = R_j$ for all $i,j \in [K]$, with 
	$R_{1}=R_{2}=R_{3}= R_{\max} = 3$, and (2) when the covariances of the mixture components have non-identical ranks, i.e., $R_i \neq R_j$ for some $i,j$, with $ (R_{1},R_{2},R_{3}) = (1,3,4), R_{\max}=5$. We generate $N = 10^5$ observations in dimension $d = 100$ from a $K =3$ mixture with mixing probabilities $\bpi^\star = (0.2,0.3,0.5)$. The highest moment order is $L = 4$. Ground-truth mixtures are generated in the following steps: 
	\begin{itemize}
		\label{list:ground-truth-conditions}
		\item
		${\bmu}_j^\star \sim \mathrm{Unif}\left(\{x \in \R^d : \norm{x}_2 = 1\}\right)$,
		\item $\Sigma_j^\star = U_j^\star \Lambda_j^\star {U_j^\star}^\top$, where $U_j^\star$ is a $d\times R_j$ random orthonormal matrix and $\Lambda_j^\star = \diag\left(\left({\lambda_j^{(1)}}^\star\right)^2, \dots, \left({\lambda_j^{(R_j)}}^\star\right)^2\right)$, ${\lambda_j^{(1)}}^\star, \dots, {\lambda_j^{(R_j)}}^\star \sim \mathrm{Unif}(\lambda_{\min}, \lambda_{\max})$ and $\lambda_{\min}, \lambda_{\max} \gg 1$. We use $\lambda_{\min} = 5, \lambda_{\max} = 10$ to enforce the low-SNR condition in \cref{def:average-lowrank-snr}. \cref{tab:datasets} reports the empirical $\SNR$ of the ground-truth mixtures.
	\end{itemize}
	Ten different random initializations are generated in the following steps: 
	\begin{itemize}
		\label{list:initialization-conditions}
		\item $\pi^{[0]}_j = \frac{1}{K}$ for all $j$, that is, the initial guess is a uniform mixture,
		\item ${\bmu}^{[0]}_j \sim \mathrm{Unif}\left(\{x \in \R^d : \norm{x}_2 = 1\}\right)$,
		\item $\Sigma^{[0]}_j = U^{[0]}_j {U^{[0]}_j}^\top$, where $U^{[0]}_j$ is a $d\times R_{\max}$ random orthonormal matrix. 
	\end{itemize} 
	
	In both sets of experiments, we test standard MM and GMM on the same parameter estimation task. To isolate the effect of weighting, we compare DGMM adaptive weights against identity weights and augmented-system weights in \cite{pereira2022tensor}, using the same data, initializations, and optimization framework in \cref{eq:d-gmom-multistep} except for the weights:
	\begin{enumerate}
		\item \textbf{Standard MM:} solving \cref{eq:mm-optimization} with explicit moment computation.
		\item \textbf{Standard GMM:} solving \cref{eq:o-gmom-multistep} with explicit moment computation.
		\item \textbf{DGMM framework with identity weights:} solving multi-step procedure in  \cref{eq:d-gmom-multistep} with implicit moment computation and fixed weights $\widehat{w}_1^{[t]}  = \cdots = \widehat{w}_L^{[t]} = 1$ for all $t$. 
		\item \textbf{DGMM framework with \cite{pereira2022tensor}'s augmented-system weights:} solving multi-step procedure in  \cref{eq:d-gmom-multistep} with implicit moment computation and fixed weights $\widehat{w}_k^{[t]} = \binom{L}{k} \omega^{2(L-k)}$ for all $t$ and $k = 1, \dots, L$, where $\omega = 0.5$ following the choice in \cite[Section 7]{pereira2022tensor}. In particular, the set of weights from \cite{pereira2022tensor} is $(0.0625,0.375,1,1)$ for $L = 4$. 
		\item \textbf{DGMM framework with the proposed adaptive weights:} solving multi-step procedure in  \cref{eq:d-gmom-multistep} with implicit moment computation and adaptive weights $\widehat{w}_k^{[t]} =   \frac{\Tr\left[ \left(\widehat{S}^{[t]}\right)^{(k, k)}\right]}{\sum_{k^\prime}\norm{\left(\widehat{S}^{[t]}\right)^{(k, k^\prime)}}_F^2}$. DGMM weights are re-estimated from $\widehat{S}^{[t]}$ at each step and in these two low-SNR experiments,  $\widehat{w}_1^{[t]} > \widehat{w}_3^{[t]} > \widehat{w}_2^{[t]} > \widehat{w}_4^{[t]}$, decreasing within each parity class. Odd-order weights are larger than the even-order weights since $\widehat{w}_k^{[t]}$ is approximately the inverse sampling variance of the noisiest fourth-order conditions, divided by their squared correlation with the order-$k$ conditions, which is weaker for odd $k$. 
	\end{enumerate}
	Across all methods, we set the stopping criteria for the multi-step estimation: the maximum number of estimation steps $T=10$ and the convergence tolerance $\varepsilon_{{\btheta}}= 10^{-4}$, whichever is met first. Within each step, we re-parameterize the constrained optimization problem in \cref{eq:d-gmom-multistep} to obtain an equivalent unconstrained problem for computational simplicity, using the softmax function\footnote{The softmax re-parameterization and regularization parameter are standard machine learning practices, for which the theoretical justifications can be found in prior works including \cite{bishop1994mixture, bishop1995neural, bishop2006pattern, hinton2015distilling, jang2016categorical, maddison2016concrete}.} $\pi_j = \frac{\exp\left\{(\widetilde{\pi}_j-\widetilde{\pi}_{\max})/\tau\right\}}{\sum_{j^\prime=1}^K \exp\left\{(\widetilde{\pi}_{j^\prime}-\widetilde{\pi}_{\max})/\tau\right\}}$, where $\tau$ is a regularization parameter that improves numerical stability and convergence behavior by controlling how ``peaked'' the distribution of the mixing probabilities becomes after the softmax re-parameterization. We then solve the unconstrained problem using L-BFGS \cite{liu1989limited} with the analytic gradients in \cref{eq:gradients-efficient} and the maximum number of L-BFGS iterations $I = 300$. All experiments are run on a 2020 MacBook Pro (Apple M1 processor, 16GB RAM).
	\begin{table}[H]
		\small
		\captionsetup{justification=justified, singlelinecheck=off}
		\centering
		
		\begin{tabular}[c]{|M{0.18\linewidth}|M{0.04\linewidth}|M{0.07\linewidth}|M{0.02\linewidth}|M{0.02\linewidth}| M{0.08\linewidth}|M{0.07\linewidth}|M{0.13\linewidth}|M{0.09\linewidth}|}
			\cline{1-9}
			\textbf{Case }& $d$ & $N$ & $K$ & $L$ & \textbf{True ranks }& \textbf{Fitted $R_{\max}$} & $\bpi^\star$ & \textbf{SNR}\\
			\cline{1-9}
			Identical covariance ranks & 100 & 100000 & 3 & 4 & $(3,3,3)$ & 3  & $(0.2, 0.3, 0.5)$ & 0.023728\\
			\cline{1-9}
			Non-identical covariance ranks  & 100 & 100000 & 3 & 4 & $(1,3,4)$ & 5 & $(0.2, 0.3, 0.5)$ & 0.019681\\
			\cline{1-9}
		\end{tabular}
		\caption{\small Summary of the experimental set-ups.}
		\label{tab:datasets}
	\end{table}
	
	The metrics for comparison are (1) computational performance measured by the average runtime and total L-BFGS iterations, and (2) statistical performance measured by the average relative estimation error with the permutation $
	\widehat\sigma\in\argmin_{\sigma\in S_K}
	\frac1K\sum_{j=1}^K
	\frac{\|\widehat\Sigma_j-\Sigma_{\sigma(j)}^{\star}\|_F}{\|\Sigma_{\sigma(j)}^{\star}\|_F}$:
	\begin{align}
		\label{eq:error-metrics}
		\err_{\pi}=\frac1K\sum_{j=1}^K \frac{|\widehat\pi_j-\pi_{\widehat\sigma(j)}^{\star}|}{|\pi_{\widehat\sigma(j)}^{\star}|},
		\err_{\mu}=\frac1K\sum_{j=1}^K \frac{\|\widehat\mu_j-\mu_{\widehat\sigma(j)}^{\star}\|_2}{\|\mu_{\widehat\sigma(j)}^{\star}\|_2},
		\err_{\Sigma}=\frac1K\sum_{j=1}^K \frac{\|\widehat\Sigma_j-\Sigma_{\widehat\sigma(j)}^{\star}\|_F}{\|\Sigma_{\widehat\sigma(j)}^{\star}\|_F}.
	\end{align} 
	
	\subsection{Results}
	\label{subsec:results}
	
	In the case of same covariance ranks, we run each method for ten different random initializations starting from fixed initialization across methods and given the number of components $K=3$ and $R_1 = R_2 = R_3 = R_{\max} = 3$. The flexibility of \cref{model:heteroscedastic} allows us to study the case when different mixture components have covariance matrices of different ranks with the maximum possible rank $R_{\max} = 5$ (individual ranks are not known a priori). Standard MM and GMM with explicit moment computation, methods~(1) and~(2), are also attempted, but are computationally infeasible to run in both cases at this scale.
	We therefore focus on examining the effects of different moment weightings in methods~(3), (4), and (5).
	
	In terms of estimation error (\cref{fig:error}), in both cases, DGMM adaptive weights lead to narrower error distributions and uniformly small parameter errors from all ten different random initializations, including the more difficult initializations from which the identity weights and the augmented-system weights of \cite{pereira2022tensor} lead to poor solutions. The parameter errors show that while favorable starts may produce comparable errors for all weights, DGMM adaptive weights make the estimation less sensitive to bad initializations than the identity weighting and the augmented-system weighting. It is worth noting that the experiments in \cite[Section~7.3]{pereira2022tensor} use unit-norm means and diagonal covariances with entries of size at most $0.4$, which gives a much larger SNR than the mixtures considered here. The numerical results suggest that difficult initializations may be more common at low SNR, making statistically principled moment weighting particularly crucial in this regime.
		
		In terms of runtime (\cref{fig:runtime}), in both cases, over the ten trials with different random initializations, the DGMM adaptive weights require fewer total L-BFGS iterations for \cref{eq:d-gmom-multistep}, thus leading to a faster median runtime than the identity weights and augmented-system weights. This observed improvement in the convergence behavior suggests that the DGMM adaptive weights not only improve the statistical efficiency in theory but also improve the numerical conditioning during the multi-step GMM optimization by providing the appropriate scaling for each moment order.
		
	\begin{figure}[H]
		\captionsetup{justification=justified, singlelinecheck=off}
		\centering
		\includegraphics[width=\textwidth]{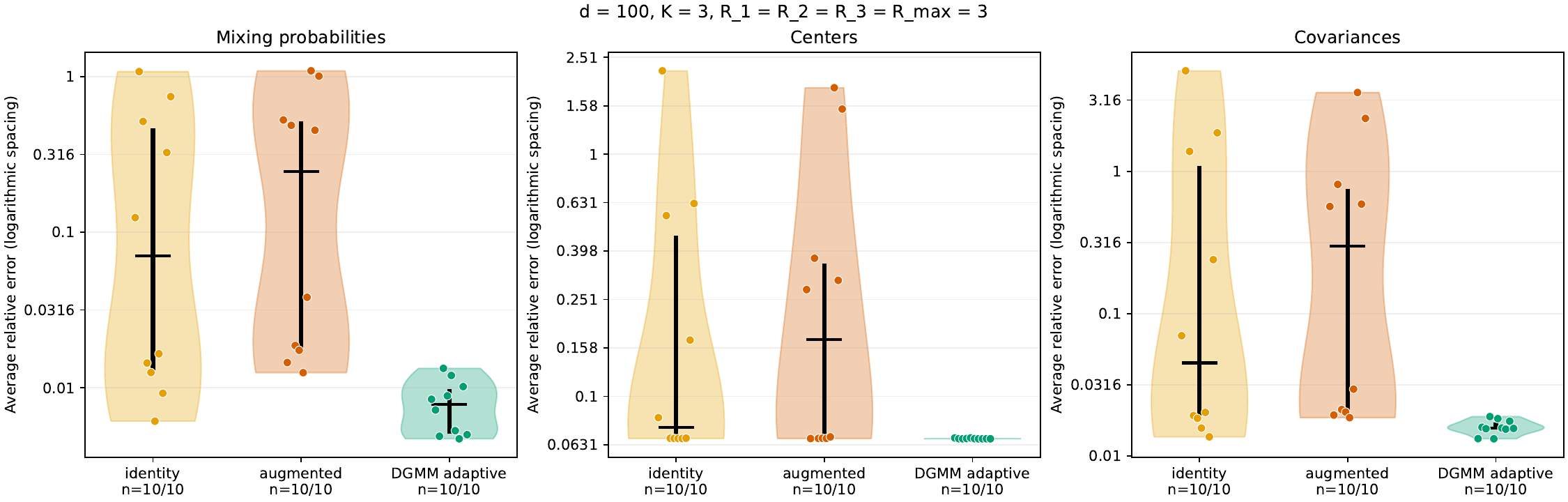}
		\includegraphics[width=\textwidth]{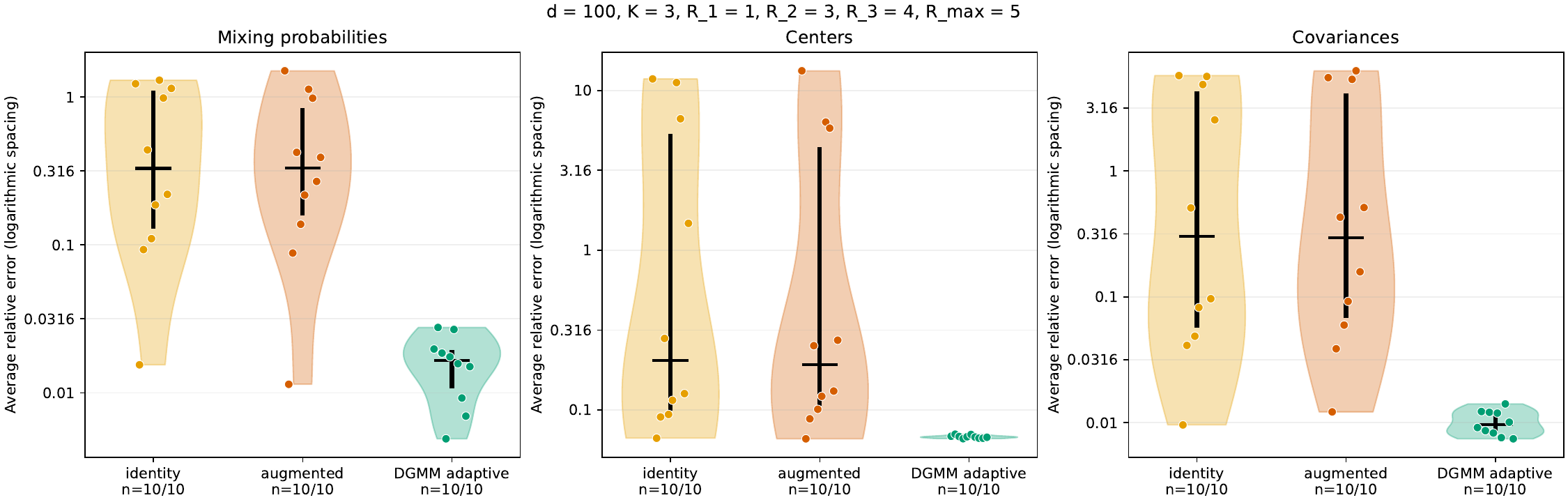}
		\caption{\small  Distribution of average relative parameter errors $\err_{\bpi}$, $\err_{\bmu}$, and $\err_{\Sigma}$ in \cref{eq:error-metrics} in logarithmic scale for DGMM with identity, augmented-system, and adaptive weights (methods~(3), (4), and (5)), from the ten different random initializations on one fixed data set, under Case 1 identical covariance ranks (above) and Case 2 non-identical covariance ranks (below). Dots represent individual trials; shaded regions are kernel density estimates of the logarithms of the errors; black vertical bars and horizontal ticks indicate interquartile ranges and medians. }\label{fig:error}
	\end{figure}
	
	\begin{figure}[H]
		\captionsetup{justification=justified, singlelinecheck=off}
		\centering
		\includegraphics[width=0.49\textwidth]{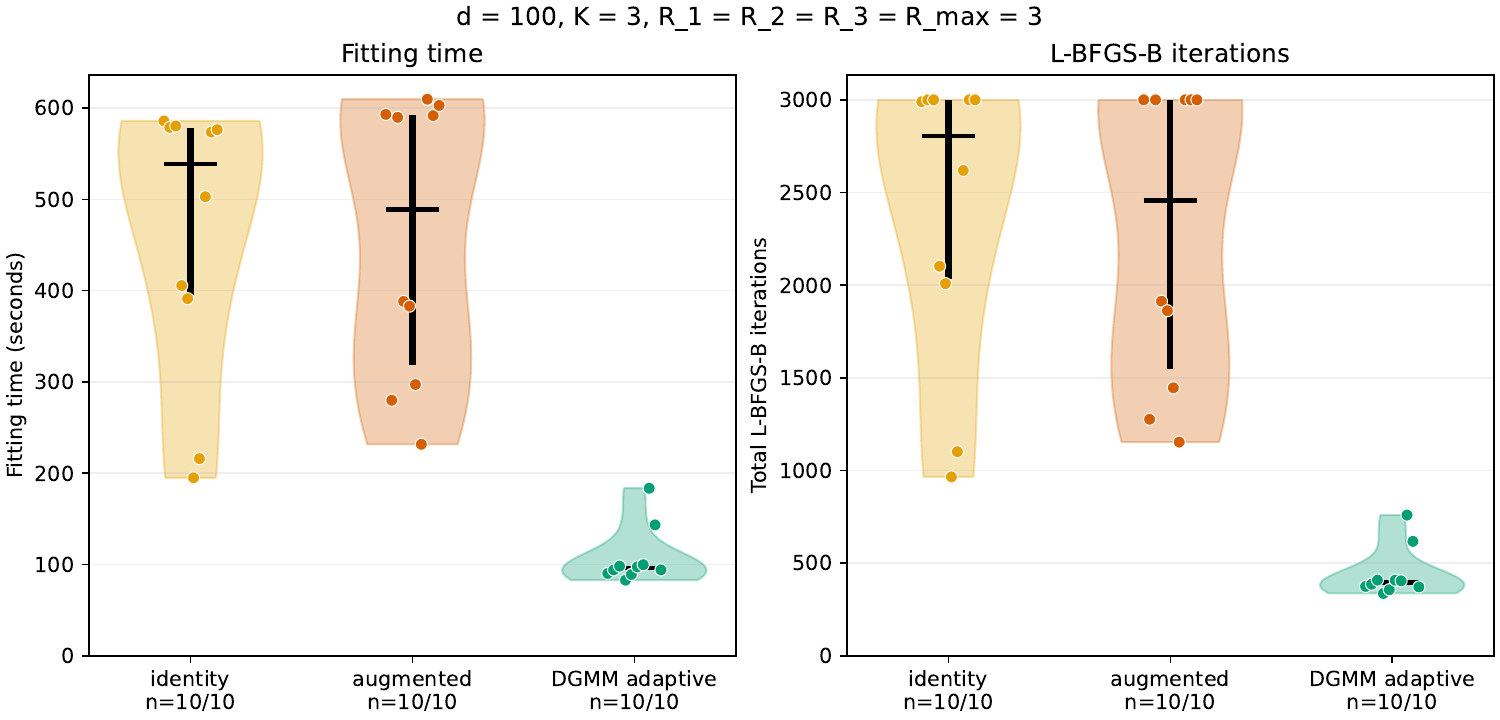}
		\includegraphics[width=0.49\textwidth]{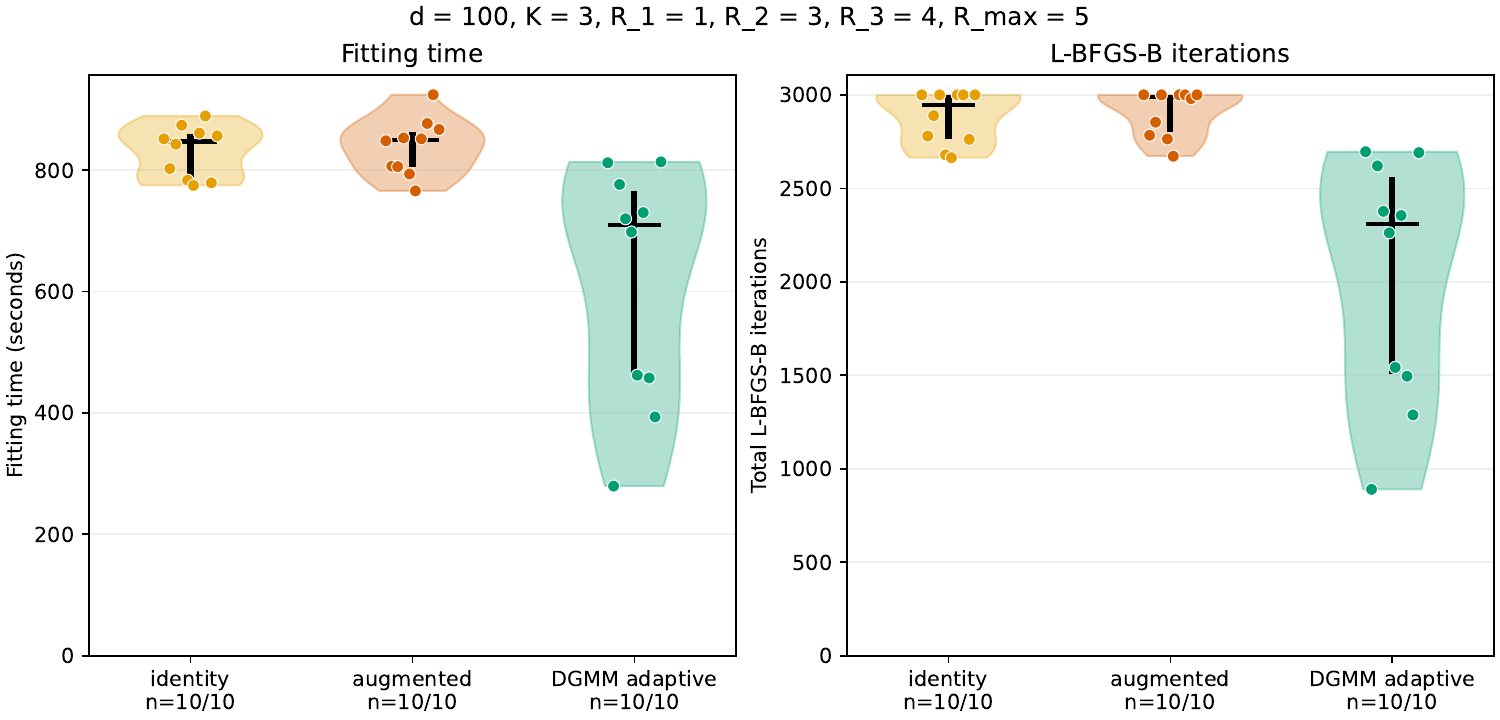}
		\caption{\small Distribution of fitting runtimes in seconds and total number of L-BFGS iterations summed over all estimation steps capped at $TI = 3000$, under case 1 identical covariance ranks (left) and case 2 non-identical covariance ranks (right). Methods, initializations, and plot conventions as in \cref{fig:error}. }\label{fig:runtime}
	\end{figure}

	\section{Conclusion}
	\label{sec:conclusion}
	We proposed the diagonally-weighted GMM (DGMM) to address the challenge of balancing statistical efficiency, computational complexity, and numerical stability in moment-based estimation. The key insight was to use a diagonal weighting matrix that assigns less weight to noisier sample moment conditions, thereby improving the estimation accuracy relative to the unweighted counterpart.  We developed this insight further by deriving the statistical properties of the DGMM estimator in \cref{thm:statistical-properties}, showing that it is consistent and asymptotically normal and near-efficient relative to optimally-weighted GMM under explicit structural conditions in \cref{thm:near-efficiency}.  Notably, this weighting strategy is conceptually aligned with recent results on likelihood and moment expansion in related low-SNR problems, e.g., \cite{abas2022generalized, katsevich2023likelihood, fan2023likelihood, fan2024maximum}, reinforcing the theoretical soundness of our approach.
	
	To a more practical end, we provided a computationally efficient and numerically stable algorithm for obtaining the DGMM estimator in the context of low-SNR heteroscedastic low-rank Gaussian mixture (\cref{model:heteroscedastic}). Our algorithm leveraged two main techniques to improve the computational complexity. First, by exploiting the relation between moments, cumulants, and Bell polynomials, we avoided explicitly computing or storing the moment tensors. This significantly lowered the cost of computing weights and gradients, effectively extending the approach of \cite{pereira2022tensor} to the settings of \cref{model:heteroscedastic} under a weighted moment-matching objective function. Second, we applied Nystr\"om approximation with $k$-means++ landmarks \cite{oglic2017nystrom} to efficiently approximate the inner-product kernel used in weight computation, with a provably bounded expected approximation error for kernel $k$-means++ landmarks (\cref{lemma:approximation-error-kernel}). In the numerical studies, DGMM with adaptive weights attained accurate estimates from all ten different random initializations and required fewer iterations and runtime, whereas the standard MM and GMM procedures were infeasible to run and alternative weighting schemes under the same solver settings were sensitive to the initialization.
	
	While we have chosen to focus on applying DGMM for Gaussian mixture modeling due to its ubiquitous applications, the utility of DGMM extends to more general settings. In fact, the advantages of DGMM discussed in \Cref{sec:proposed-dgmom} hold for general parameter estimation problems and the statistical properties of DGMM hold as long as the model of interest satisfies the global and local identification assumptions and the standard regularity conditions\footnote{See standard references, e.g., \cite{newey1994large} for a comprehensive discussion on this.}. More broadly, our results support DGMM weighting as a tractable approach to the statistical and computational tradeoffs of moment-based low-SNR mixture estimation. The algorithm for obtaining DGMM, while tailored to \cref{model:heteroscedastic}, is designed and implemented with adaptability in mind and can be modified for other models. Given its theoretical properties and empirical performance, we expect that DGMM will be useful for a broad class of statistical estimation problems, especially in challenging low-SNR regimes.
	
	\section*{Acknowledgments} 
	We would like to thank Afonso S. Bandeira, Tamir Bendory, Charles Fefferman, Marc A. Gilles, Anya Katsevich, Joe Kileel, João M. Pereira, and Liza Rebrova for helpful discussions. In particular, we thank Marc A. Gilles for suggesting the randomly pivoted Cholesky algorithm to speed up the code and João M. Pereira for helpful suggestions on numerical optimization. We would also like to acknowledge that the anonymous reviews have greatly improved the manuscript. Claude Opus 5.5 was used to assist in developing a wrapper for repeated trials of the existing implementation of the algorithms and in writing code to plot the numerical results. The authors assume responsibility for all content.
	
	\section*{Funding}
	The work was supported in part by AFOSR FA9550-20-1-0266 and FA9550-23-1-0249, the Simons Foundation Math+X Investigator Award, NSF DMS 2009753, and NIH/NIGMS R01GM136780-01.
	
	\bibliographystyle{plain}
	\bibliography{dgmm_references}
	
	\appendix
	
	\section{Supporting proofs for model specifications}
	\label{sec:supporting-proofs-model}
	Suppose the covariance of the $j$-th Gaussian component in \cref{model:heteroscedastic}, $\Sigma_j$, has the compact singular value decomposition:
	\begin{align}
		\label{eq:simga-compact-svd}
		\Sigma_j = U_{j} \Lambda_{j} U_{j}^\top, \mkern9mu U_{j} \in \R^{d\times R_j}, \mkern9mu \Lambda_j = \diag\left({\lambda_j^{(1)}}, \dots, {\lambda_j^{(R_j)}}\right) \in \mathbb{R}^{R_j\times R_j}.
	\end{align}
	Due to the rank deficiency of $\Sigma_j$, the p.d.f.~of ${\bX}_j$ does not exist w.r.t.~the $d$-dimensional Lebesgue measure $\nu_d$, as shown in the following proposition:
	\begin{proposition}
		\label{prop:measure-d-R}
		The p.d.f.~of ${\bX} \sim \cN({\bmu},\Sigma) \in \R^d$ where $\rank \Sigma = R < d$ does not exist w.r.t.~the $d$-dimensional Lebesgue measure $\nu_d$.
	\end{proposition}
	\begin{proof}
		For ${\bX} \sim \cN({\bmu},\Sigma)$ with singular $\Sigma$ , $\nu_d({\bmu} + \range(\Sigma)) =0$. For any measurable subset $A$ of the affine space ${\bmu} + \range(\Sigma)$, we have $\nu_d(A) \leq \nu_d({\bmu} + \range(\Sigma))$. Thus, $\nu_d(A) = 0$. As shown by \cite[(21.12.5) on p. 290]{cramer1946mathematical}, \cite[(8a.4.11) and (8a.4.12) on p. 527-528]{rao1973multivariate} , and \cite[p.43]{srivastava1979introduction}, if $\Sigma$ is singular, then ${\bX} \sim \cN ({\bmu}, \Sigma)$ lies on a $\rank \Sigma$-dimensional subspace of $\R^d$, almost surely. This implies that there exists some measurable subset $A \subseteq {\bmu} + \range(\Sigma)$ such that $\probP[{\bX}\in A] \neq 0$. By the Radon-Nikod\'ym theorem, the p.d.f.~of ${\bX}$ does not exist w.r.t.~$\nu_d$. 
	\end{proof}
	However, the p.d.f.~of ${\bX}_j$ does exist w.r.t.~the following $R_j$-dimensional measure supported on the affine space ${\bmu}_j +\range(\Sigma_j)$, as shown in the following proposition:
	\begin{proposition}[\cite{khatri1968some}]\label{prop:pdf-single-gaussian-lowrank}
		Let ${\bX} \sim \cN({\bmu},\Sigma)$ with  $\rank \Sigma = R$. Suppose $\Sigma$ has the compact SVD: $\Sigma = U_R \Lambda_R U_R^\top$, with $U_R \in \R^{d\times R}$ and $\Lambda_R= \diag\left({\lambda^{(1)}}, \dots, {\lambda^{(R)}}\right) \in \mathbb{R}^{R\times R}$.
		Let $N \in \R^{d\times (d-R)}$ be a matrix of rank $(d-R)$ such that $N^\top \Sigma = 0$. Define the affine map $\phi: \R^R \to \R^d, \phi(\cdot) \coloneqq {\bmu} + U_R (\cdot)$. Let $\nu_R: \scrB(\R^R) \to \R_{\geq 0}$ denote the $R$-dimensional Lebesgue measure, where $\scrB(\R^d)$ denotes the standard Borel algebra of $\R^d$. Then the p.d.f.~of ${\bX}$ exists w.r.t.~the following $R$-dimensional measure supported on the affine space ${\bmu}+\range(\Sigma)$:
		\begin{equation}
			\rho: \scrB(\R^d)\to \R_{\geq 0}, \mkern9mu \rho \coloneqq \nu_R \circ \phi^{-1},
		\end{equation}
		and the p.d.f.~is given by $p_{{\bX}}(\bx) = 
		\frac{\exp\left\{-\frac{1}{2} (x- {\bmu})^\top \Sigma^{\dagger} (x-{\bmu})\right\}}{(2\pi)^{\frac{R}{2}} \left(\Det(\Sigma)\right)^{\frac{1}{2}}}$ for $x \in {\bmu} + \range(\Sigma))$, where $\Sigma_j^{\dagger}$ is Moore-Penrose generalized inverse, $\Det(\Sigma_j) =  \prod_{r=1}^{R_j} \lambda_j^{(r)}$ is the pseudo-determinant (see \cite{knill2014cauchy}). 
	\end{proposition}
	
	\section{Supporting proofs for computational complexity and statistical properties}
	\label{sec:supporting-proofs-statistical-computational}
	\printProofs
	
\end{document}